\pdfoutput=1
\documentclass[twoside,11pt]{article}

\DeclareRobustCommand\onedot{\futurelet\@let@token\@onedot}
\def\@onedot{\ifx\@let@token.\else.\null\fi\xspace}

\usepackage[export]{adjustbox}
\usepackage{xcolor}
\usepackage{float}

\usepackage{subcaption}

\usepackage[abbrvbib, preprint]{jmlr2e}

\usepackage{multirow}
\usepackage{booktabs}

\usepackage[linesnumbered,ruled,vlined]{algorithm2e}

\SetCommentSty{mycommfont}
\makeatletter
\def\algbackskip{\hskip-\ALG@thistlm}
\makeatother

\jmlrheading{2}{2020}{1-48}{4/00}{10/00}{meila00a}{Farshid Varno, Lucas May Petry, Lisa Di Jorio and Stan Matwin.\\
Portions of the technology described in this paper are patent pending}

\ShortHeadings{Fast And Stable Task-adaptation (FAST)}{Varno, Petry, Di Jorio and Matwin}
\firstpageno{1}

\usepackage{amsmath,amsfonts,bm}

\def\eqref#1{equation~(\ref{#1})}
\def\Eqref#1{Equation~(\ref{#1})}
\def\EqrefNP#1{Equation~\ref{#1}}

\def\1{\bm{1}}

\def\vzero{{\bm{0}}}
\def\vone{{\bm{1}}}

\def\vtheta{{\bm{\theta}}}
\def\va{{\bm{a}}}

\def\vu{{\bm{u}}}
\def\vv{{\bm{v}}}

\def\vx{{\bm{x}}}
\def\vy{{\bm{y}}}
\def\vz{{\bm{z}}}
\def\vdelta{{\bm{\delta}}}

\def\evtheta{{\theta}}
\def\evdelta{{\delta}}
\def\eva{{a}}

\def\evl{{l}}

\def\evy{{y}}
\def\evz{{z}}

\def\mA{{\bm{A}}}

\def\mV{{\bm{V}}}
\def\mW{{\bm{W}}}
\def\mX{{\bm{X}}}
\def\mY{{\bm{Y}}}
\def\mZ{{\bm{Z}}}

\def\mDelta{{\bm{\Delta}}}
\DeclareMathAlphabet{\mathsfit}{\encodingdefault}{\sfdefault}{m}{sl}
\SetMathAlphabet{\mathsfit}{bold}{\encodingdefault}{\sfdefault}{bx}{n}

\def\gM{{\mathcal{M}}}

\def\gS{{\mathcal{S}}}
\def\gT{{\mathcal{T}}}

\def\gV{{\mathcal{V}}}

\def\gX{{\mathcal{X}}}

\def\emW{{W}}

\newcommand{\E}{\mathbb{E}}

\newcommand{\R}{\mathbb{R}}

\newcommand{\lr}{\alpha}

\begin{document}

\title{
Learn Faster and Forget Slower via Fast and Stable Task Adaptation
}

\author{\name Farshid Varno \email f.varno@dal.ca \\
       \addr Department of Computer Science\\
       Dalhousie University\\
       6050 University Ave, Halifax, NS B3H 1W5, Canada
      \AND
      \name Lucas May Petry \email lucas.petry@posgrad.ufsc.br \\
      \addr Programa de Pós-Graduação em Ciência da Computação\\ 
      Universidade Federal de Santa Catarina\\
      Florianópolis, SC, Brazil
      \AND
      \name Lisa Di Jorio \email{lisa@imagia.com} \\
      \addr Imagia Cybernetics Inc. \\
      6650 St Urbain St \#100, Montreal, Quebec H2S 3G9, Canada
      \AND
      \name Stan Matwin \email{stan@cs.dal.ca}\\
      \addr Department of Computer Science\\
      Dalhousie University\\
      6050 University Ave, Halifax, NS B3H 1W5, Canada
       }

\editor{Francis Bach and David Blei
and Bernhard Sch{\"o}lkopf}

\maketitle

\begin{abstract}
Training Deep Neural Networks (DNNs) is still highly time-consuming and compute-intensive. 
It has been shown that adapting a pretrained model may significantly accelerate this process. With a focus on classification, we show that current fine-tuning techniques make the pretrained models catastrophically forget the transferred knowledge even before anything about the new task is learned. Such rapid knowledge loss undermines the merits of transfer learning and may result in a much slower convergence rate compared to when the maximum amount of knowledge is exploited. We investigate the source of this problem from different perspectives and to alleviate it, introduce Fast And Stable Task-adaptation (FAST), an easy to apply fine-tuning algorithm. The paper provides a novel geometric perspective on how the loss landscape of source and target tasks are linked in different transfer learning strategies. We empirically show that compared to prevailing fine-tuning practices, FAST learns the target task faster and forgets the source task slower. 

\end{abstract}


\section{Introduction}
In many real-world applications, learning quickly is as vital as the ultimate performance. However, the speed at which the results are required is often not able to be predetermined. For example, a smartphone user may expect a face or speech recognition tool to be customized and provide reasonable results at anytime after exposure to the data. 
Ideally, a fast converging \textit{anytime algorithm} is needed to progressively learn these kind of tasks. \textit{Anytime algorithms} are expected to improve their performance upon allocating more computation time and data~\citep{korf1998a} and inevitably could be interrupted at anytime depending on how impatient the user or application is~\citep{boddy1989solving}.
Conversely, a typical 
Deep Neural Network (DNN) model with so-called medium capacity may require up to weeks or even months to be trained using thousands of training examples and large amounts of computational power until it exhibits a satisfactory level of performance.\footnote{As of today, a DNN model with tens to hundreds of millions of parameters is known to have medium capacity.}
Moreover, since first-order Gradient Descent (GD)---the preferred means of training DNNs---is not an {anytime algorithm}, guaranteeing fast and progressive generalizable results while training is challenging for DNNs, making them harder to deploy at any time.

A popular trend currently defines the life-cycle of a DNN model to eventually settle in the \textit{Model Deployment} stage and work in that stage indefinitely~\citep{ashmore2019assuring}. With this mindset, although the progression of the performance 
is usually validated using a held-out portion of the training data, different models are typically compared through their final test results. How quickly cutting-edge models reach different performance milestones during training is often overlooked, which leads to unnecessary expenses in addition to large amounts of heat emission~\citep{li2016evaluating} that in turn damages the environment in different ways~\citep{strubell2019energy}.\footnote{As an example of the training cost, \citet{yang2019xlnet} (\textit{XLNet}) reported using a cloud infrastructure for training their model that costs tens of thousands of dollars.} The importance of learning efficiency is even more pronounced when it comes to certain learning settings. For instance, a conventional practice in \textit{neural architecture search} is to find competent architectures based on the performance that each sampled architecture achieves after being trained for a relatively short time ~\citep{wang2019alphax, zela2018towards}. 

Various techniques have been proposed to ease the burden of training DNNs that can potentially offer better performance in shorter time, thus lowering the time requirements and compute cost.
These techniques include initialization~\citep{glorot2010understanding, He_2015_ICCV, arpit2019benefits}, optimization~\citep{qian1999momentum, botev2017nesterov, tieleman2012lecture, kingma2014adam}, and normalization~\citep{ioffe2015batch,ba2016layer,Wu_2018_ECCV, qiao2019weight} methods, architectural modifications~\citep{He_2016_resnet,szegedy2017inception}, and many others \citep{He_2019_CVPR}. 
A well-known approach for accelerating the training process of a DNN on a given task is to fine-tune it after being pretrained on a similar source task with abundant data~\citep{Girshick_2014_CVPR}.

Despite its high necessity, a well-defined fine-tuning procedure that leads to fast and progressively generalizable results is currently missing from the transfer learning literature. A typical practice of adopting a 
pretrained model is to remove some parameters and append new ones such that the model's output matches the dimensions of the target task~\citep{wang2019easy, chu2016best}.\footnote{We use the term \textit{model adoption} (not adaptation) to refer to the process in which the pretrained model is altered such that it becomes prepared for learning a target task.} In this paper, we argue that without special considerations in optimizing a model that is formed in this way, transfer learning may become highly inefficient, though still better than training from scratch. In fact, we show that carelessly optimizing such model can lead to catastrophic forgetting \citep{mccloskey1989catastrophic, goodfellow2013empirical}, even before learning to perform a new task.
Our objective is two-fold:
first, we want to accelerate the adaptation of Convolutional Neural Networks (CNNs) for classification;
and second, we want to maintain a good generalized performance for these models during most of the training, from the beginning until convergence. We achieve these goals mainly by adjusting the magnitude of the gradients that back-propagate through the parameters of the \textit{softmax classifier} toward 
the rest of the modules.

Our main contributions in this paper are identifying a source of inefficiency in the parametric softmax classifier, bringing to light the catastrophic forgetting it causes, and finally tackling it by introducing Fast And Stable Task-adaptation (FAST), a novel optimization procedure for fine-tuning. 
Unlike the prevailing technique of fine-tuning which is unconcerned by initialization of the appended parameters~\citep{Girshick_2014_CVPR, chu2016best, wang2019easy}, our algorithm does not entangle the spaces formed by the pretrained and appended parameters all at once; instead, it bridges them in a gradual and smooth order. At the beginning of the training when even the most-popular advanced-optimization techniques lack a clear perspective, FAST slows down updates on the pretrained parameters and at the same time it quickly adapts the appended ones starting from all zero---the origin of the parameter space they form. By moving away from the origin, the norm of the appended parameters increases. Properly adjusting the rate with which this norm is escalated, leads to updates with efficient magnitude and direction on the pretrained parameters. 
Our experiments suggest that FAST accelerates the convergence of fine-tuning deep pretrained CNNs for classification tasks. 

The rest of the paper is organized as follows. Section \ref{sec:bg} describes our notation and gives further insight about the scope of the paper. In Section~\ref{sec:geometric}, our novel geometric perspective is both introduced and exploited to depict the most common transfer learning
techniques.\footnote{The term~\textit{transfer learning} may represent a broad range of techniques which involve transferring knowledge among tasks; however, unless explicitly indicated, we use it to refer to \textit{classical transfer learning}; the most widely known category of transfer learning in which the parameters that are pretrained on one task are employed to learn another.} 
Our proposed method for fine-tuning is explained in Section~\ref{sec:method} and its relation to other research works is expanded in Section~\ref{sec:relate}. Section~\ref{sec:exp} contains the experiments and results and finally the paper is concluded in Section~\ref{sec:future}, where also future works are discussed.

\section{Background and Notation}
\label{sec:bg}
In this section, our notation is establish and a brief background on parametric classifiers is provided.

\subsection{Notation}
In our setting, a CNN consists of a non-linear feature extractor $\phi: \gX \longrightarrow \R^{D}$ that is followed by a linear label-space mapping,\footnote{Each instance in the extracted feature space can be considered to be 1-dimensional. Practically, this unification in number of dimensions of extracted features is applied in two major ways for image classification, described by~\citep{Lifchitz_2019_CVPR}.} $f: \R^{D} \longrightarrow \R^{N}$ and a probability estimation function $p$.
$\phi$ extracts the features of any input $\vx \in \gX$ in the form of a feature vector such as $\va \in \R^D$.
We assume each true label $\vy$ is one-hot encoded, logistic softmax is used as the probability estimation function, meaning
\begin{equation}
    \label{eq:sm}
    p(\vz) = \frac{e^{\vz}}{e^{\vz}\vone_{N\times N}},
\end{equation}
and the negative log-likelihood is employed as the loss function.\footnote{We could not find any previous work which defines the softmax in a vectorized form as we present in~\Eqref{eq:sm}, so we assume that this notation for softmax is novel.} If the output of $f(\va)$ is $\vz$ and $\hat{\vy} = p(\vz)$, the loss could be written as
\begin{equation}
    \label{eq:loss_objective}
    l = - \vy \ln{(\hat{\vy})}.
\end{equation}
The elements of vectors $\vz \in \R^N$ and $\hat{\vy} \in \R^N$ are called softmax logits and predicted labels, respectively.

In this paper, we use the term \textit{CNN classifier} to refer to a model which contains a convolutional feature extractor, a label-space mapping module, and a probability estimation function. For disambiguation, the term \textit{classifier} will be used to refer to a model with similar components excluding a feature extractor.
The exponential functions and division in \Eqref{eq:sm} apply element-wise on their vector argument and $\vone_{N \times N}$ is an $N \times N$ matrix of all ones.
For the sake of simplicity and consistency, we refer to the mini-batched version of the forwarding data, the backwarding gradients and the labels with capital symbols. For example, $\mA \in \R^{M\times D}$ is the mini-batched version of $\va$ with $M$ rows, each containing feature vectors corresponding to one example.
The elements of {matrices} are addressed by subscripts with the first element indicating the row number. $\mA_m$ and $\mA_{:,d}$ represent the $m$-th row and $d$-th column of $\mA$, respectively.
We also show the sequence of training mini-batches exposed to the model by writing the number of applied gradient updates inside parenthesis in the superscripts. For instance, $\mZ^{(0)}$ refers to the logits corresponding to the first iteration of the training mini-batch. 
In our notation, $||\vu||_v$ and $||\vu||_F$ represent the L$_v$-norm and Frobenius norm of vector $\vu$ respectively; whereas, $|\vu|$ indicates its number of elements or dimensions (e.g. if $\vu \in \R^V$ then $|\vu| = V$).

\subsection{Parametric Classifiers}
The importance of the  classifier's role has been already discussed in DNN literature. However, most of the research works in this realm have only focused on the sort of features that different classifiers encourage the feature extractor to learn~\citep{Luo2018cosine,chen2019a_closer}.  
The randomly-initialized 
parametric softmax classifier is still the preferred classifier of state-of-the-art models in computer vision~\citep{He_2016_resnet, Huang_2017, Ma_2018_ECCV, Sandler_2018_CVPR, Xie_2017_CVPR, tan2019efficientnet}.
It applies a softmax function on a linear mapping $f$, where $f$ is often the dot-product between its inputs (extracted features) and its parameters $\mW \in \R^{N \times D}$, that is 
\begin{equation}
    \label{eq:fc_op}
    f_{\text{dp}}(\va, \mW) = \va \ \mW^T,
\end{equation}
in which the subscript $\text{db}$ stands for \textit{dot-product}.

The first dimension of $\mW$ is the number of classes, so its dimension depends on the performing task. To adapt a pretrained model to a new task, usually a random set of values is drawn for initializing $\mW$ \citep{Girshick_2014_CVPR, agrawal2014analyzing, li2018learning}. Appending a randomly-initialized classifier to a meaningful set of features results in a heterogeneously parametrized model. Training a pretrained feature extractor along with a randomly initialized classifier can contaminate the genuinely learned representations and significantly degrade the maximum transferable knowledge. 

Usually the fine-tuning is slowed down by decreasing the learning rate to compensate for this knowledge leak~\citep{li2018learning} which undermines the applications' potential for early deployment.\footnote{\citet{li2018learning} refer to the knowledge that is leaked from the feature-extractor's parameters at the beginning of the fine-tuning as \textit{drift in the parameters}.} Such a downturn in efficiency is often disregarded, though some attempts have been made to improve the performance through modifying the classifier for certain models, learning scenarios and learning objectives. For instance, ~\cite{Luo2018cosine} proposed some modifications to a \textsc{VGG} model~\citep{Simonyan2014vgg} along with replacing the dot-product in the label-space mapping (see~\EqrefNP{eq:fc_op}) with two candidate operations. Similarly many studies in \textit{few-shot learning} replace the commonly used dot-product operation of the classifier with another bounded operation so called \textit{cosine classifier}~\citep{chen2019a_closer, vinyals2016matching, gidaris2018dynamic}.\footnote{The mapping function of cosine classifier is generally defined as
$     \forall n \in \{1,\dots, N\}: f_{\text{cosine}} (\va, \mW_n) = \gamma( \frac{\va}{{||\va||}_2} \frac{\mW_n^T}{{||\mW_n||}_2})
$,
where $\gamma$ is a constant or learnable scalar.}
While our focus in this paper is only on a typical learning scenario, we look for the source of inefficiency in the initialization of the $\mW$ instead of making changes to the projection operation. We justify our perspective by highlighting the forgetting effect during the fine-tuning of a pretrained model. In our proposed method, the euclidean norms of the feature-to-label-space projection vectors ($||\mW_n||_2$) increase during the course of training. This gradual growth empowers the model to smoothly adapt the pretrained parameters to the new task with a reduced forgetting effect.

From the definition of $f_{db}$ in ~\Eqref{eq:fc_op} and the loss objective noted in \Eqref{eq:loss_objective}, it is easy to show that
\begin{equation}
    \label{eq:grad_aNz}
    \frac{\partial l}{\partial \vz} = - \vdelta = \hat{\vy} - \vy \ ; \ \ \frac{\partial l}{\partial \va} = - \vdelta \ \mW.
\end{equation}
The prediction error (or simply error) $\vdelta$ is non-positive everywhere, except for the index corresponding to the true class. Notice that if the index of the true class is $n$ then $\vy_n=1$.
The updates on $\mW$ are calculated based on the cumulative gradient of all examples in the mini-batch. Therefore, using the batched version of $\va$ and $\vdelta$ we have
\begin{equation*}
    \frac{\partial \ell}{\partial \mW} = - \frac{1}{M}\mDelta^T \mA,
\end{equation*}
in which $\frac{1}{M}$ appears from averaging the loss across the batch dimension. Here, it is assumed that to calculate the mini-batch's loss, the dimension of the loss vector is reduced by averaging---i.e., $\ell =\frac{1}{M} \sum_{m=1}^{M} \evl_m$ where $\evl_m = -\mY_m \ln \hat{\mY}_m$.
For simplicity, let us assume vanilla gradient descent with learning rate $\lr$ is applied, then the $(\tau+1)$-th update for $\mW$ is calculated as
\begin{equation}
\label{eq:recursive_w_fc}
\begin{split}
    \mW^{(\tau+1)} &= \mW^{(\tau)} + \frac{\lr}{M} {\mDelta^T}^{(\tau)} \mA^{(\tau)}\\
    &= \mW^{(0)} + \frac{\lr}{M} \sum_{t=1}^{\tau} {\mDelta^T}^{(t)} \mA^{(t)}
\end{split}
\end{equation}
where $\mW^{(0)}$ denotes the value of $\mW$ before applying any updates. Traditionally, the elements of $\mW^{(0)}$ are randomly selected, yet there have not been an agreement nor a notable argument on the distribution to draw these elements from. Different strategies chosen for initializing $\mW$ are further discussed in Section\ref{sec:related_init}.

\section{Proposed Geometric Interpretation of Transfer Learning}
\label{sec:geometric}
To justify the impact of adopting a pretrained model, most research works rely on the concept of learning \textit{multiple levels of representations} \citep{zeiler2013stochastic, Goodfellow2016book}. According to this concept, higher-level representations contain more \textit{abstract features} \citep{bengio2012deep}. These representations correspond to the layers that are closer to the input of the model and usually are the ones that are maintained through the model adoption process \citep{zhong2016face}.

Another popular geometric view of DNNs, depicts the loss landscape in the space formed by all the model's parameters, regardless of the layers they belong to~\citep{li2018visualizing}. This perspective 
has provided researchers with the intuitions behind many optimization algorithms \citep{qian1999momentum, duchi2011adaptive, tieleman2012lecture, kingma2014adam}. 
Stochastic Gradient Descent (SGD) walks the parameters through the parameters' space under the light of local clues in the loss landscape to settle them in a minimum point, and hopefully this point provides the estimator a small loss when it is exposed to unseen data.

In order to employ the latter geometric view for understanding the concept of transfer learning, answering a fundamental question seems to be necessary:
\textit{in the space formed by the model's parameters, how much and in what direction should SGD drive the parameters to quickly adapt a pretrained model to a target task?} In other words, in the aforementioned space, \textit{how much the minimum of the target task's loss  deviates from that of the source task and what is a good trajectory to quickly approach it?}

On the way to search for an answer, let the euclidean space formed by the parameters $\vu$ be $\gS(\vu) = \gS(\vu_1, \vu_2, \dots, \vu_{|\vu|})$. Notice that, each point on the landscape of the loss that is depicted in parameter space $\gS(\vu)$, is in fact a point in $\gS(\vu, \ell)$ with the loss ($\ell$) being the only non-parametric dimension in this space. As mentioned earlier, to adopt a pretrained model for a classification task, the space of all model's parameters are often manually changed from $\gS(\vtheta,\vv)$ to $\gS(\vtheta,\mW)$ and we call $\vv$ and $\mW$ removed and appended parameters, respectively. So, during the adoption process, not only the number of dimensions in the parameter space may change, but also the appended parameters are generally not relevant to either of the tasks (the source task and the target task) before fine-tuning starts. In fact,
right after appending new parameters, spotting the minimum of the source task's loss---where the pretraining has landed in---becomes intractable. 

To make it possible to exploit the geometric view of the loss landscape for transfer learning, we only consider loss landscape in $\gS(\vtheta)$. This relaxed view enables our analysis to link the loss landscapes of the two tasks at the expense of having more complicated view of the loss landscape of the target task. In this scheme, the modifications made in the appended parameters are reflected as deformation of the loss landscape in the space of the pretrained parameters. More formally, moving in $\gS(\mW)$ is reflected as changes along the dimension $\ell$ of $\gS(\vtheta, \ell)$. This perspective is flexible enough to symbolically describe the effect of different transfer learning strategies as it is shown in Figure \ref{fig:space_classic} and Figure \ref{fig:space_alternative}. 

The intuition provided in these figures is novel in the sense that it can reflect the modifications in $\gS(\vtheta)$ in an abstract way while having an idea about how the loss landscape is reflected in that space. It relaxes the complexity in depicting a heterogeneous space formed by both the pretrained and appended parameters, and opens up the opportunity to geometrically link the source task's and target task's loss landscapes. 

In each of the subplots in Figure \ref{fig:space_classic} or Figure \ref{fig:space_alternative}, a model pretrained on task $\gT_s$ is adopted to accelerate learning the target task $\gT_t$. The center of the gray contours at the bottom, marked with $\vtheta^*_{\gT_s}$, represents the minimum in the loss landscape of $\gT_s$ where $\vtheta$ settles in, at the end of the pretraining. Similarly, the aiming minimum in the loss landscape of $\gT_t$ is shown with $\vtheta^*_{\gT_t}$, though its location in the space is subjected to change because the appended parameters are modified by the optimization algorithm. Thus, to make a clearer view, changes in the location of $\vtheta^*_{\gT_t}$  is shown by multiple sets of contours with a grey-level set at the top to indicate the initial location and is following by several color-tempered sets which represent the deformations made in the loss landscape. we have simplified the landscape deformations with simple affine transformations (shift, scale and rotation) applied on the minimum; however, many other kinds of transformations are also possible. The levels of the loss landscapes are shown with differently tempered colors. The black is the coldest color used to show the minima and ivory is used to describe the high altitude surfaces (with large loss values). Notice that, the chosen shapes are symbolic to simply show the effect of each method; otherwise, the true patterns of minima could be of any arbitrary shape \citep{skorokhodov2019loss, czarnecki2019deep}.

Starting from $\vtheta = \vtheta^*_{\gT_s}$, the goal of the optimization is to make $\vtheta$ get as close as possible to $\vtheta^*_{\gT_t}$. Figure \ref{fig:space_fe} describes feature extraction (e.g., employed by~\citet{donahue2014decaf}) in which only the appended parameters are trained and thus $\vtheta = \vtheta^*_{\gT_s}$ stays hold during the task-adaptation.
On the other hand in fine-tuning (e.g., employed by \citet{Girshick_2014_CVPR}) shown in Figure \ref{fig:space_ft}, both transferred and appended parameters are jointly trained. The blue curve represents the movements of $\vtheta$ in $\gS(\vtheta)$ corresponding to the steps taken by the optimization algorithm. 

As shown in Figure~\ref{fig:space_classic}, compared to feature extraction, fine-tuning can potentially guide $\vtheta$ closer to $\vtheta_{\gT_t}^*$. 
However, as depicted in Figure \ref{fig:space_ft} and described later in the paper, carelessly initializing $\mW$ can mislead the optimization algorithm to guide $\vtheta$ in a wrong direction at the beginning of fine-tuning.
If the initial steps be large, they take $\vtheta$ far away from $\vtheta_{\gT_t}^*$ and make the training slow. One possible solution employed by \citet{li2018learning} is to start with a \textit{classifier warmup} phase within which only $\mW$ is modified and $\vtheta$ is kept unchanged before they jointly be train as in the typical fine-tuning. This approach is depicted in Figure 
\ref{fig:space_wup} where the arrows and circled numbers determine the sequence of the applied changes.

The inclusion of a classifier warmup phase prevents $\mW$  from randomly distorting
the back-propagating gradients and, therefore, the initial modifications on $\vtheta$ become more likely to be in a path that leads to $\vtheta_{\gT_t}^*$.
However, since the distance from  $\vtheta_{\gT_t}^*$ is not predetermined by the first-order GD, it is also likely that the initial steps for $\vtheta$ overshoot the minimum. Clearly, it is not easy to find an optimal step size for these steps considering the commonly large number of dimensions in $\gS(\vtheta, \ell)$. As we will expand in the next sections, our proposed method for fine-tuning builds $\vtheta_{\gT_t}^*$ close to $\vtheta_{\gT_s}^*$ without notably misdirecting $\vtheta$ at anytime from beginning until the convergence. The geometrically interpretation of this method is shown in Figure \ref{fig:space_fast}.

\subsection{Factors that Control the Velocity of Convergence}

Ideally, SGD drives the pretrained parameters ($\vtheta$) toward a good minimum in the target task's loss landscape located at $\vtheta_{\gT_t}^*$, and at the same time, through modifying $\mW$ it guides $\vtheta_{\gT_t}^*$ toward $\vtheta$ in the opposite direction. 
The velocity of moving $\mW$ in $\gS(\mW)$ at optimization step $\tau$ is calculated as
\begin{equation}
    \label{eq:v_w_1}
    \mV_{\mW}^{(\tau)} = \mW^{(\tau+1)} - \mW^{(\tau)}.
\end{equation}
On the other hand, \Eqref{eq:recursive_w_fc} could be written in the following form
\begin{equation}
    \label{eq:w_diff}
    \mW^{(\tau+1)} = \mW^{(\tau)} +  \frac{\alpha}{M} {\mDelta^{T}}^{(\tau)} \mA^{(\tau)},
\end{equation}
which can turn \Eqref{eq:v_w_1} 
into
\begin{equation}
    \label{eq:v_w}
    \mV_{\mW}^{(\tau)} = \frac{\alpha}{M} {\mDelta^{T}}^{(\tau)} \mA^{(\tau)}.
\end{equation}
Moreover, the relative velocity of $\mW$ and $\ell$ is equal to $\mV_{\mW,\ell}=\frac{\partial\ell}{\partial \mW}$ which is already found in \Eqref{eq:grad_aNz}. Interestingly,
\begin{equation}
    \mV_{\mW,\ell} = \frac{1}{\alpha} \mV_{\mW},
\end{equation}
which is equal to stating that the velocity of deformations of the loss landscape---and in the extreme case, the velocity of shifting $\vtheta_{\gT_t}^*$---in $\gS(\vtheta)$  that is rooted from modifying $\mW$ is linearly related to what we found in \Eqref{eq:v_w}. 

Similarly, the velocity of moving $\vtheta$ in $\gS(\vtheta)$ at optimization step $\tau$ is
\begin{equation}
    \label{eq:velocity_simple}
    \mV_{\vtheta}^{(\tau)} = \vtheta^{(\tau+1)} - \vtheta^{(\tau)}.
\end{equation}
Since SGD is used to update the parameters, we have
\begin{equation}
    \label{eq:v_theta_before_plugged}
    \vtheta^{(\tau+1)} = \vtheta^{(\tau)} - \frac{\beta}{M} \left(\frac{\partial \ell^{(\tau)}}{\partial \mA^{(\tau)}}\right)^T \frac{\partial \mA^{(\tau)}}{\partial \vtheta^{(\tau)}},
\end{equation}
where $\beta$ is the learning rate with which $\vtheta$ is updated. Notice that in Section~\ref{sec:bg}, $\alpha$ is employed to represent the learning rate that SGD uses for updating $\mW$ and traditionally $\alpha=\beta$. By plugging the formula we obtained for calculating $\mW^{(\tau)}$ in Section~\ref{sec:bg}, \Eqref{eq:v_theta_before_plugged} can be updated to
\begin{equation}
\label{eq:v_theta_plugged}
    \mV_{\vtheta}^{(\tau)} = \frac{\beta}{M} \ \left(\mDelta^{(\tau)} \mW^{(\tau)}\right)^T \ \frac{\partial \mA^{(t)}}{\partial \vtheta^{(\tau)}}.
\end{equation}
From \Eqref{eq:recursive_w_fc} and \Eqref{eq:v_theta_plugged} we can conduct
\begin{equation}
    \label{eq:v_theta_and_w_plugged}
    \mV_{\vtheta}^{(\tau)} = \frac{\beta}{M} \left({\mW}^{(0)} + \frac{\alpha}{M} \sum_{t=1}^{\tau-1} {\mDelta^T}^{(t)} {\mA}^{(t)} \right)^T {\mDelta^T}^{(\tau)} \frac{\partial \mA^{(\tau)}}{\partial \vtheta^{(\tau)}},
\end{equation}
or even simpler
\begin{equation*}
    \mV_{\vtheta}^{(\tau)} = \frac{\beta}{M} \left(\sum_{t=0}^{\tau-1} \mV_{\mW}^{(t)} \right)^T {\mDelta^T}^{(\tau)} \frac{\partial \mA^{(\tau)}}{\partial \vtheta^{(\tau)}}.
\end{equation*}
This shows that varying the rate with which $\mW$ is modified, not only changes the rate of deforming and shifting the aiming minimum of loss in $S(\vtheta)$, but also greatly influences how fast $\vtheta$ approaches it.
The calculated velocities take the origin of $\gS(\vtheta)$ as the \textit{privileged frame};\footnote{In other words, the origin of $\gS(\vtheta)$ is the reference of the calculated velocities} thus, according to our hypothesis the relative velocity of $\vtheta$ with respect to $\vtheta_{\gT_t}^*$ is expected to be different than $V_{\vtheta}$ (desirably larger). 
An interesting fact about \Eqref{eq:v_theta_and_w_plugged} and \Eqref{eq:v_w} is that altering $\beta$ only influences the speed of updating $\vtheta$ whereas $\alpha$ impacts the speed of shifting the aimed minimum in addition to that. As we will explain in Section~\ref{sec:method}, this fact helps FAST to provide a less compromised generalized performance during the course of training. 

\begin{figure}[t]
    \centering
    \begin{subfigure}{0.47\textwidth}
        \centering
        \def\svgwidth{7.1cm}
\begingroup%
  \makeatletter%
  \providecommand\color[2][]{%
    \errmessage{(Inkscape) Color is used for the text in Inkscape, but the package 'color.sty' is not loaded}%
    \renewcommand\color[2][]{}%
  }%
  \providecommand\transparent[1]{%
    \errmessage{(Inkscape) Transparency is used (non-zero) for the text in Inkscape, but the package 'transparent.sty' is not loaded}%
    \renewcommand\transparent[1]{}%
  }%
  \providecommand\rotatebox[2]{#2}%
  \newcommand*\fsize{\dimexpr\f@size pt\relax}%
  \newcommand*\lineheight[1]{\fontsize{\fsize}{#1\fsize}\selectfont}%
  \ifx\svgwidth\undefined%
    \setlength{\unitlength}{579.40002441bp}%
    \ifx\svgscale\undefined%
      \relax%
    \else%
      \setlength{\unitlength}{\unitlength * \real{\svgscale}}%
    \fi%
  \else%
    \setlength{\unitlength}{\svgwidth}%
  \fi%
  \global\let\svgwidth\undefined%
  \global\let\svgscale\undefined%
  \makeatother%
  \begin{picture}(1,0.97514666)%
    \lineheight{1}%
    \setlength\tabcolsep{0pt}%
    \put(0,0){\includegraphics[width=\unitlength,page=1]{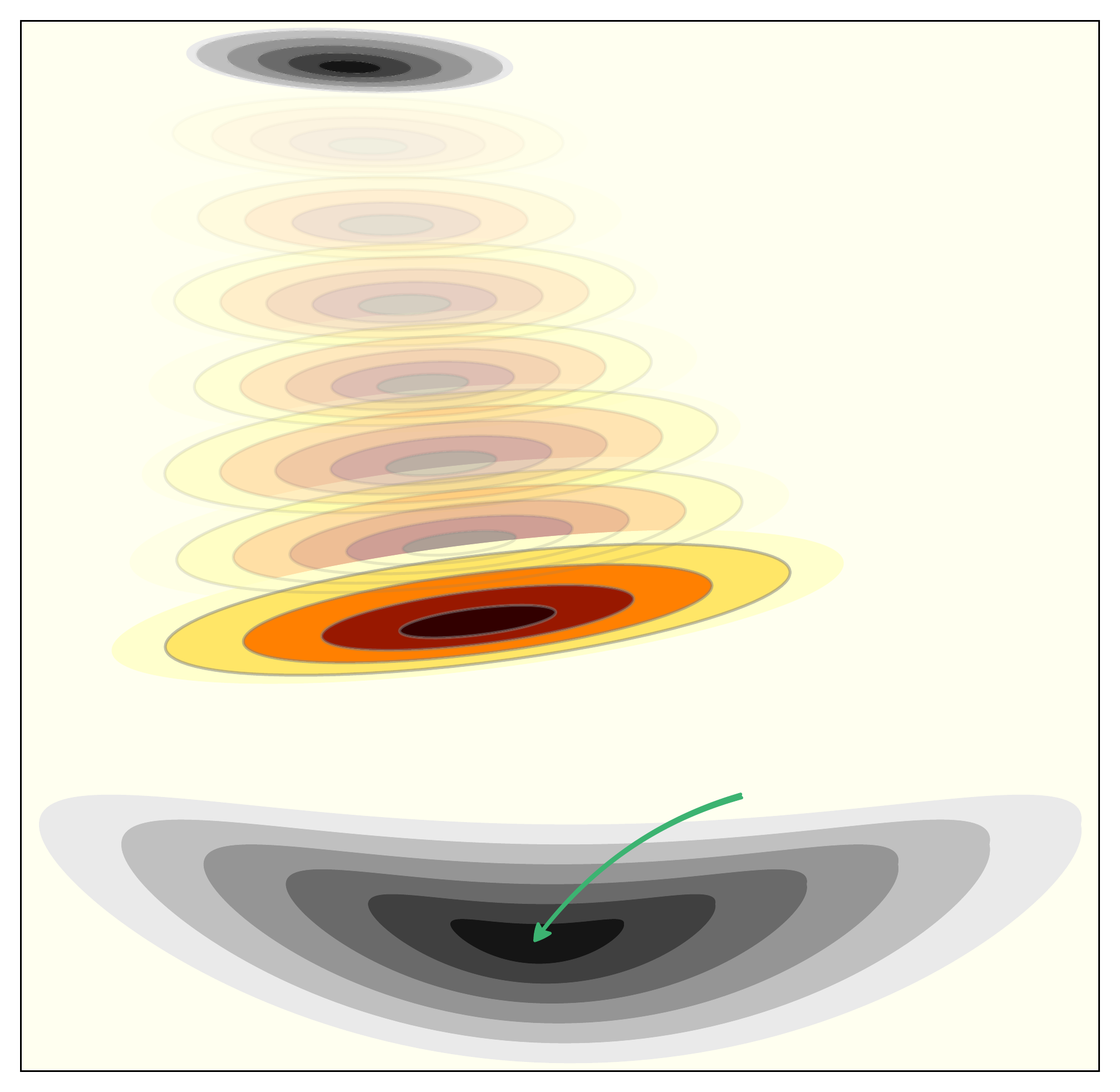}}%
    \put(0.6481076,0.27026389){\makebox(0,0)[lt]{\lineheight{1.25}\smash{\begin{tabular}[t]{l}\ $\vtheta^*_{\gT_s}$\end{tabular}}}}%
    \put(0,0){\includegraphics[width=\unitlength,page=2]{space_fe.pdf}}%
    \put(0.40928662,0.33020048){\makebox(0,0)[lt]{\lineheight{1.25}\smash{\begin{tabular}[t]{l}\ $\vtheta^*_{\gT_t}$\end{tabular}}}}%
  \end{picture}%
\endgroup%

        \caption{Feature extraction}
        \label{fig:space_fe}
    \end{subfigure}
    \begin{subfigure}{0.47\textwidth}
        \centering
        \def\svgwidth{7.1cm}
\begingroup%
  \makeatletter%
  \providecommand\color[2][]{%
    \errmessage{(Inkscape) Color is used for the text in Inkscape, but the package 'color.sty' is not loaded}%
    \renewcommand\color[2][]{}%
  }%
  \providecommand\transparent[1]{%
    \errmessage{(Inkscape) Transparency is used (non-zero) for the text in Inkscape, but the package 'transparent.sty' is not loaded}%
    \renewcommand\transparent[1]{}%
  }%
  \providecommand\rotatebox[2]{#2}%
  \newcommand*\fsize{\dimexpr\f@size pt\relax}%
  \newcommand*\lineheight[1]{\fontsize{\fsize}{#1\fsize}\selectfont}%
  \ifx\svgwidth\undefined%
    \setlength{\unitlength}{579.40002441bp}%
    \ifx\svgscale\undefined%
      \relax%
    \else%
      \setlength{\unitlength}{\unitlength * \real{\svgscale}}%
    \fi%
  \else%
    \setlength{\unitlength}{\svgwidth}%
  \fi%
  \global\let\svgwidth\undefined%
  \global\let\svgscale\undefined%
  \makeatother%
  \begin{picture}(1,0.97514666)%
    \lineheight{1}%
    \setlength\tabcolsep{0pt}%
    \put(0,0){\includegraphics[width=\unitlength,page=1]{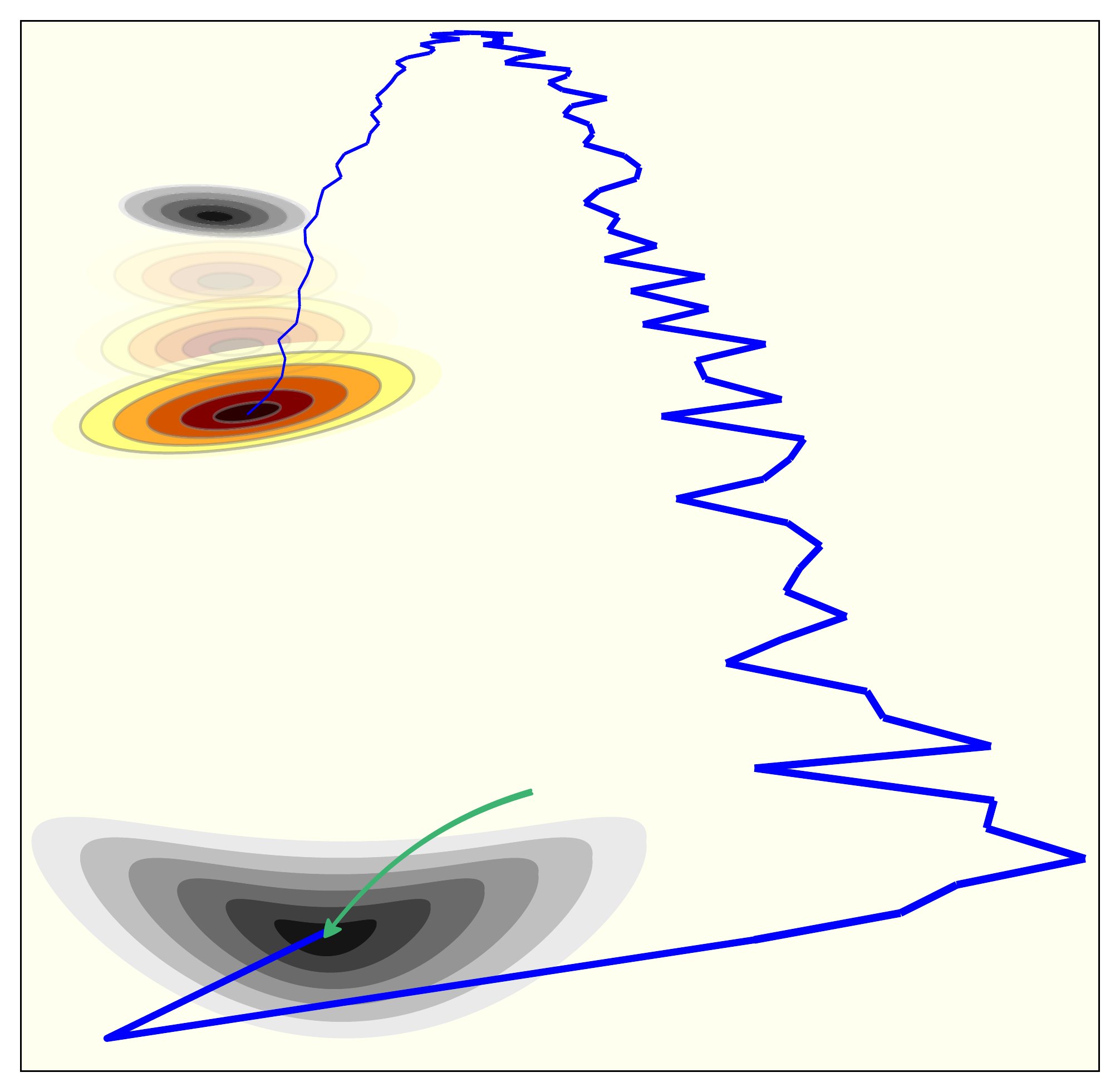}}%
    \put(0.46071798,0.27381771){\makebox(0,0)[lt]{\lineheight{1.25}\smash{\begin{tabular}[t]{l}\ $\vtheta^*_{\gT_s}$\end{tabular}}}}%
    \put(0,0){\includegraphics[width=\unitlength,page=2]{space_ft.pdf}}%
    \put(0.20345184,0.51855367){\makebox(0,0)[lt]{\lineheight{1.25}\smash{\begin{tabular}[t]{l}\ $\vtheta^*_{\gT_t}$\end{tabular}}}}%
  \end{picture}%
\endgroup%

        \caption{Fine-tuning}
        \label{fig:space_ft}
    \end{subfigure}
    \caption{
    The loss landscapes and optimization steps for adapting a hypothetical pretrained model to a target task in the geometric space of transferred parameters using the classical transfer learning methods.
    }
    \label{fig:space_classic}
\end{figure}

\begin{figure}[t]
    \centering
    \begin{subfigure}{0.47\textwidth}
        \centering
        \def\svgwidth{7.1cm}
\begingroup%
  \makeatletter%
  \providecommand\color[2][]{%
    \errmessage{(Inkscape) Color is used for the text in Inkscape, but the package 'color.sty' is not loaded}%
    \renewcommand\color[2][]{}%
  }%
  \providecommand\transparent[1]{%
    \errmessage{(Inkscape) Transparency is used (non-zero) for the text in Inkscape, but the package 'transparent.sty' is not loaded}%
    \renewcommand\transparent[1]{}%
  }%
  \providecommand\rotatebox[2]{#2}%
  \newcommand*\fsize{\dimexpr\f@size pt\relax}%
  \newcommand*\lineheight[1]{\fontsize{\fsize}{#1\fsize}\selectfont}%
  \ifx\svgwidth\undefined%
    \setlength{\unitlength}{579.40002441bp}%
    \ifx\svgscale\undefined%
      \relax%
    \else%
      \setlength{\unitlength}{\unitlength * \real{\svgscale}}%
    \fi%
  \else%
    \setlength{\unitlength}{\svgwidth}%
  \fi%
  \global\let\svgwidth\undefined%
  \global\let\svgscale\undefined%
  \makeatother%
  \begin{picture}(1,0.97514666)%
    \lineheight{1}%
    \setlength\tabcolsep{0pt}%
    \put(0,0){\includegraphics[width=\unitlength,page=1]{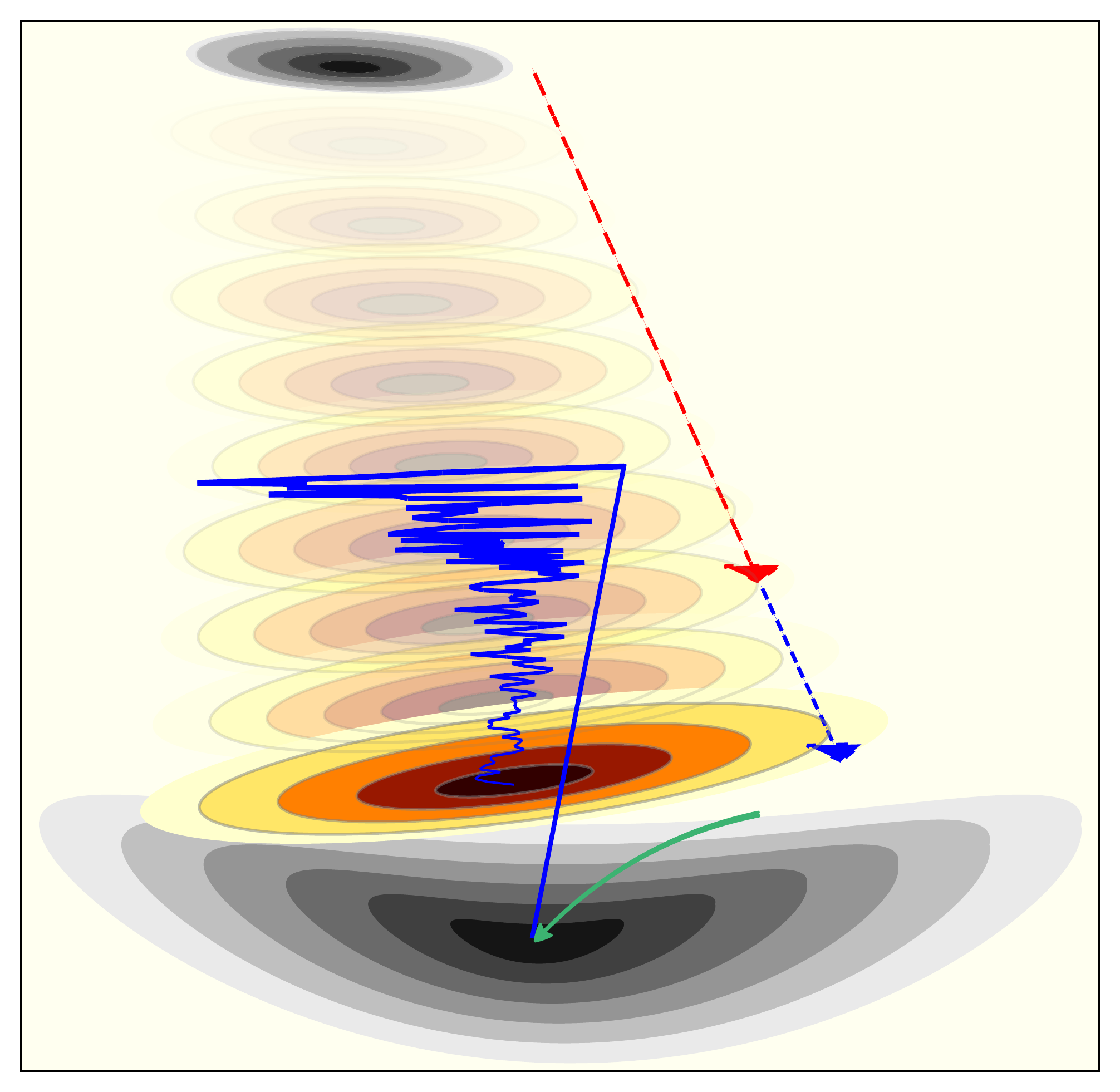}}%
    \put(0.68262606,0.25300465){\makebox(0,0)[lt]{\lineheight{1.25}\smash{\begin{tabular}[t]{l}\ $\vtheta^*_{\gT_s}$\end{tabular}}}}%
    \put(0,0){\includegraphics[width=\unitlength,page=2]{space_wup.pdf}}%
    \put(0.13126671,0.08449174){\makebox(0,0)[lt]{\lineheight{1.25}\smash{\begin{tabular}[t]{l}\ $\vtheta^*_{\gT_t}$\end{tabular}}}}%
    \put(0,0){\includegraphics[width=\unitlength,page=3]{space_wup.pdf}}%
    \put(0.6387467,0.70080331){\color[rgb]{1,1,1}\makebox(0,0)[lt]{\lineheight{1.25}\smash{\begin{tabular}[t]{l}1\end{tabular}}}}%
    \put(0,0){\includegraphics[width=\unitlength,page=4]{space_wup.pdf}}%
    \put(0.58977726,0.34541998){\color[rgb]{1,1,1}\makebox(0,0)[lt]{\lineheight{1.25}\smash{\begin{tabular}[t]{l}2\end{tabular}}}}%
  \end{picture}%
\endgroup%

        \caption{Fine-tuning with classifier warmup (as employed by~\citet{li2018learning})}
        \label{fig:space_wup}
    \end{subfigure}
    \begin{subfigure}{0.47\textwidth}
        \centering
        \def\svgwidth{7.1cm}
\begingroup%
  \makeatletter%
  \providecommand\color[2][]{%
    \errmessage{(Inkscape) Color is used for the text in Inkscape, but the package 'color.sty' is not loaded}%
    \renewcommand\color[2][]{}%
  }%
  \providecommand\transparent[1]{%
    \errmessage{(Inkscape) Transparency is used (non-zero) for the text in Inkscape, but the package 'transparent.sty' is not loaded}%
    \renewcommand\transparent[1]{}%
  }%
  \providecommand\rotatebox[2]{#2}%
  \newcommand*\fsize{\dimexpr\f@size pt\relax}%
  \newcommand*\lineheight[1]{\fontsize{\fsize}{#1\fsize}\selectfont}%
  \ifx\svgwidth\undefined%
    \setlength{\unitlength}{579.40002441bp}%
    \ifx\svgscale\undefined%
      \relax%
    \else%
      \setlength{\unitlength}{\unitlength * \real{\svgscale}}%
    \fi%
  \else%
    \setlength{\unitlength}{\svgwidth}%
  \fi%
  \global\let\svgwidth\undefined%
  \global\let\svgscale\undefined%
  \makeatother%
  \begin{picture}(1,0.97514666)%
    \lineheight{1}%
    \setlength\tabcolsep{0pt}%
    \put(0,0){\includegraphics[width=\unitlength,page=1]{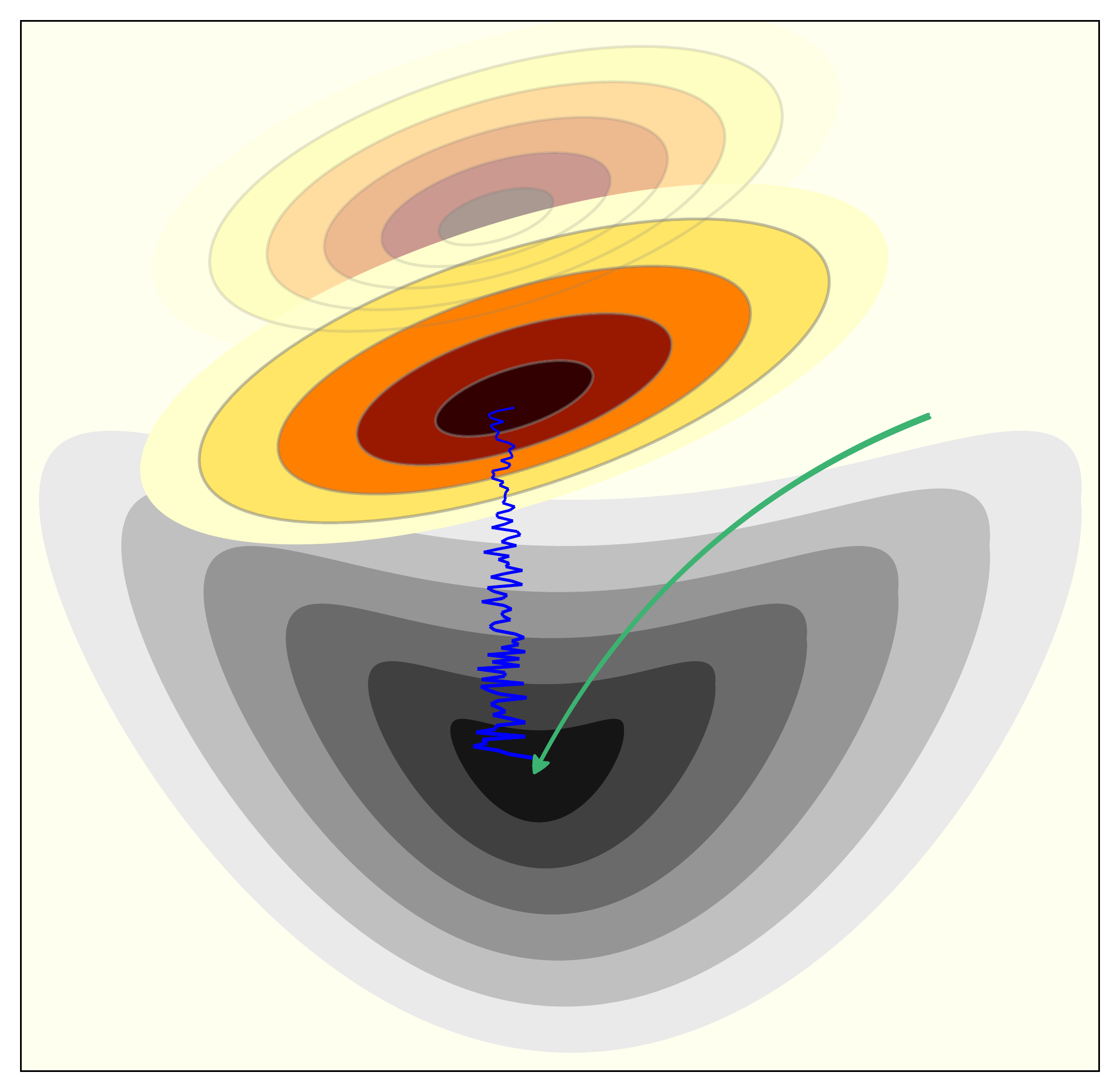}}%
    \put(0.8034407,0.60975053){\makebox(0,0)[lt]{\lineheight{1.25}\smash{\begin{tabular}[t]{l}\ $\vtheta^*_{\gT_s}$\end{tabular}}}}%
    \put(0,0){\includegraphics[width=\unitlength,page=2]{space_fast.pdf}}%
    \put(0.13126671,0.12776371){\makebox(0,0)[lt]{\lineheight{1.25}\smash{\begin{tabular}[t]{l}\ $\vtheta^*_{\gT_t}$\end{tabular}}}}%
  \end{picture}%
\endgroup%

        \caption{FAST (ours) \\ \ \ }
        \label{fig:space_fast}
    \end{subfigure}
    \caption{
    A symbolic comparison between our proposed method for fine-tuning and fine-tuning with classifier warmup in the geometric space of transferred parameters.
    }
    \label{fig:space_alternative}
\end{figure}

\subsection{Minimum Makeup}
Logistic softmax function normalizes its inputs such that they sum up to one. The magnitude of its inputs is exponentially reflected in the discrepancy among its outputs. The negative log-likelihood loss that is applied on top of it, is notably affected by this characteristic. By initializing the classifier's weights to values that are closer to zero, the loss becomes further independent from the extracted features. Equivalently, as the entries of $\mW^{(0)}$ go closer to zero, the landscape of the loss in $\vtheta$ becomes flatter, until a point where everything is leveled and equal to $\ln{(N)}$. {In this situation the loss becomes independent from the extracted features}. That is, for all possible values of $\vtheta$ in $\gS(\vtheta, \ell)$ we have $\ell=0$. For this reason, the initial state of the minima in the target task's loss landscape is not shown with gray colors in Figure \ref{fig:space_fast} unlike other plots in Figures \ref{fig:space_classic} and \ref{fig:space_alternative}. Because of its level of freedom, a flat loss landscape in the space of $\vtheta$ is much easier to deform. 
If done carefully, it can be immediately deformed such that a minimum is placed close to the current state of $\vtheta$ (this can be implied from the arguments in Section~\ref{sec:method_init}). In fact, instead of finding a solution to increase the relative speed between $\vtheta$ and $\vtheta^{*}_{\gT_t}$, this strategy makes it possible to initially build a minima close to where $\vtheta$ is located in $\gS(\vtheta)$. In Section~\ref{sec:method}, we describe how FAST uses this characteristic to decrease the number of steps required to converge.


\section{FAST}
\label{sec:method}
Fine-tuning a pretrained CNN model on a new classification task usually requires the parametric classifier to be rebuilt accordingly. The most common way of doing this is to simply replace the parameters of the pretrained classifier with randomly-initialized ones that satisfy the dimension of the target task~\citep{chu2016best}.
We argue that, optimizing the reconstructed model carelessly, 
can significantly contribute to a phenomenon known as catastrophic forgetting~\citep{goodfellow2013empirical,mccloskey1989catastrophic}. 
We show that a large $||\mW||_F$ can make the model highly biased toward certain classes at the beginning of the fine-tuning procedure.
Not only could this kind of prejudgement misrepresent the usability of the features for the new task, but it could also damage the transferred knowledge in at least two ways. First, the gradients back-propagated toward parameters of the feature extractor include an error term that is not correlated with neither the transferred knowledge nor the target task in a meaningful way.
Second, even if such an error is in a correct direction, it is distorted (weighted) by the random-initialized parameters of the classifier, which in turn affects the pretrained parameters during back-propagation. The outcome of these two issues is symbolically reflected in the unguided initial direction of the blue curve in Figure \ref{fig:space_ft}.

Depending on how largely the feature extractor's parameters are updated initially, the transferred knowledge may be highly forgotten due to massive changes in the relations among the pretrained parameters. Forgetting the previously learned knowledge before learning anything about the new task is clearly in contrast with the essence of transfer learning. The most immediate consequence of such prevailing detriment is that it can delay the convergence. As discussed in Section~\ref{sec:related_tl}, classifier warmup employed by \cite{li2018learning}, can alleviate the problems caused by initializing $\mW$ randomly. However, as shown in Section~\ref{sec:exp_wup}, the model is still prone to catastrophic forgetting right after the classifier's warmup phase ends. In this case, updates to the parameters of the feature extractor are likely in a correct direction but with an uncontrolled magnitude, which considering how the first order GD works, does not completely prevent catastrophic forgetting and may overshoot the minimum it is aiming to converge to. 
Figure \ref{fig:space_wup} intuitively reflects the effect of classifier warmup. In this figure, the red arrow marked with a circled number is plotted to show that through classifier warmup phase a minimum goes close to $\vtheta^{(0)}$ while $\beta=0$. Also, the blue circled number indicates that SGD starts modifying the pretrained parameters right after the first phase is finished.

Our proposed method, FAST, solves the aforementioned problems and hence accelerates the training process. FAST consists of two simple modifications compared to the traditional fine-tuning:
\begin{enumerate}
    \item $\mW$ is initialized with zero (see \EqrefNP{eq:ENTAME_init}).
    \item Learning rates $\alpha$ and $\beta$ are found such that $\vtheta$ and $\gT^*_2$ approach each other in $\gS(\vtheta)$ with a relative velocity that is small enough so that they do not broadly cross each other, but large enough so that the anytime performance is not compromised.
    
\end{enumerate}
Basically, FAST makes the fine-tuning procedure more robust by dampening large and noisy updates to the parameters. 
Algorithm~\ref{alg:sgd_FAST} contrasts our method with the prevalent approach described in Algorithm \ref{alg:sgd_trad}. In these algorithms, SGD with momentum~\citep{qian1999momentum} is used to optimize the objective; however, as we will see in Section~\ref{sec:exp_opt}, FAST also could help the stability and speed of the convergence when more complex optimization algorithms like Adam~\citep{kingma2014adam} is used.

In Algorithm \ref{alg:sgd_FAST}, rates of applying updates on the classifier and the feature extractor are separately chosen. In the first look this seems more challenging because of having an extra hyper-parameter but we will show the necessity of separately choosing $\alpha$ and $\beta$ in Sections \ref{sec:method_lr} and \ref{sec:exp_lr}. In the rest of this section, we justify and/or provide intuitions behind the modifications FAST makes compared to the traditional fine-tuning.   

\begin{algorithm}[t]
    \DontPrintSemicolon
    \caption{Traditional way of fine-tuning a pretrained DNN on a classification task with logistic softmax classifier using SGD+momentum}
    \label{alg:sgd_trad}
    \SetKwInOut{Input}{input}
    \SetKwInOut{Output}{output}
    \Input{$\vtheta, \ \beta, \mX, \ D, \ N$ and $\gM$}
    \Output{$\vtheta$, and $\mW$}
    Generate new parameters $\mW \in \R^{N\times D}$\;
    $t = 0, \ \gV^{(0)}_{\mW}=\gV^{(0)}_{\vtheta}=0$, $\alpha \gets \beta$\;
    Initialize the classifier s.t. $\E[\mW] = \epsilon\ ; \ \epsilon \longrightarrow 0$\;
    \While{Validation goal not satisfied}
  {
        \tcc{forward pass}
  		$\mA^{(t)} \gets \phi(\mX^{(t)}; \vtheta^{(t)})$\;
  		$\mZ^{(t)} \gets f(\mA^{(t)}; \mW^{(t)})$\;
  		$\hat{\mY}^{(t)} \gets \frac{\exp{(\mZ^{(t)}})}{\exp{(\mZ^{(t)})}\vone_{N\times N}}$ \; 
  		\tcc{backward pass}
  		$\mDelta^{(t)} \gets \mY^{(t)} -\hat{\mY}^{(t)}$ \;
  	    Use $\alpha$ \& $\beta$ to calculate $\mV_{\mW}^{(t)}$ and $\mV_{\vtheta}^{(t)}$ 
  	    from \Eqref{eq:v_w} and \Eqref{eq:v_theta_plugged},  respectively\;
  	    $\gV^{(t)}_{\mW} \gets \gM \gV^{(t-1)}_{\mW} + (1-\gM) \mV_{\mW}^{(t)}\ ; \  \gV^{(t)}_{\vtheta} \gets \gM \gV^{(t-1)}_{\vtheta} + (1-\gM) \mV_{\mW}^{(t)}$\;
  		$\mW^{(t+1)} \gets \mW^{(t)} + \gV^{(t)}_{\mW} $ ; $\ \vtheta^{(t+1)} \gets \vtheta^{(t)} + \gV^{(t)}_{\vtheta}$\;
  		$t \gets t+1$\;
  }
\end{algorithm}

\begin{algorithm}[t]
    \DontPrintSemicolon
    \caption{
    FAST for fine-tuning a pretrained DNN on a classification task}
    \label{alg:sgd_FAST}
    \SetKwInOut{Input}{input}
    \SetKwInOut{Output}{output}
    \Input{$\vtheta, \ \alpha, \ \beta, \mX, \ D$ and $N$}
    \Output{$\vtheta$, and $\mW$}
    Generate new parameters $\mW \in \R^{N\times D}$\;
    $t = 0, \ \gV^{(0)}_{\mW}=\gV^{(0)}_{\vtheta}=0$;
    Initialize the classifier s.t. $\E[\mW] = \E[\mW\mW^T] = \epsilon\ ; \ \epsilon \longrightarrow 0$\;
    \While{Validation goal not satisfied}
  {
        \tcc{forward pass}
  		$\mA^{(t)} \gets \phi(\mX^{(t)}; \vtheta^{(t)})$\;
  		$\mZ^{(t)} \gets f(\mA^{(t)}; \mW^{(t)})$\;
  		$\hat{\mY}^{(t)} \gets \frac{\exp{(\mZ^{(t)}})}{\exp{(\mZ^{(t)})}\vone_{N\times N}}$ \; 
  		\tcc{backward pass}
  		$\mDelta^{(t)} \gets \mY^{(t)} -\hat{\mY}^{(t)}$ \;
  		Use $\alpha$ \& $\beta$ to calculate $\mV_{\mW}^{(t)}$ and $\mV_{\vtheta}^{(t)}$ 
  	    from \Eqref{eq:v_w} and \Eqref{eq:v_theta_plugged},  respectively\;
  	    $\gV^{(t)}_{\mW} \gets \gM \gV^{(t-1)}_{\mW} + (1-\gM) \mV_{\mW}^{(t)}\ ; \  \gV^{(t)}_{\vtheta} \gets \gM \gV^{(t-1)}_{\vtheta} + (1-\gM) \mV_{\mW}^{(t)}$\;
  		$\mW^{(t+1)} \gets \mW^{(t)} + \gV^{(t)}_{\mW} $ ; $\ \vtheta^{(t+1)} \gets \vtheta^{(t)} + \gV^{(t)}_{\vtheta}$\;
  		$t \gets t+1$\;
  }
\end{algorithm}

\subsection{Initialization for Entropy Maximization}
\label{sec:method_init}
From a psychological point of view, bias toward specific groups or subjects matters much more in the earlier stages of the human's life. Children's attitudes are quite flexible as early as ages three or four~\citep{byrnes1995teacher}. However, attitudes of grown-ups are more difficult to change. \textit{Social Learning Theory} implies that prejudgment is primarily learned~\citep{bandura1977social}. Similarly, we expect a neural network to stay neutral in associating an example to different classes prior to any training. Otherwise, any initial tendency toward specific decisions could not only make learning more difficult, but also the resulting representation (i.e., extracted-features in our learning scenario) may stay tied to the initial bias. For example, if prior to training, two classes are assumed to be very entangled, without evidence in the ground truth, this false assumption may implicitly require the learned representations to increase the number of discriminating features activated for those classes. This can make the truly similar classes starve from discriminating features.

\subsubsection[Why Does Careful Initialization of the Classifier Parameters Matter?]{Why Does Careful Initialization of $\mW$ Matter?}
Depending on the magnitude of $\mW^{(0)}$, the extracted features $\mA$ and the learning rate $\alpha$, one can expect an arbitrary update number $\tau$ until the relations in $\mW$ are altered.
During this period, the gradients that are injected to the feature extractor's parameters, are scaled by the dot-product between the error and $\mW^{(0)}$.
That is, for $\tau' < \tau$, if 
\begin{equation}
\label{eq:initi_gg}
\forall n \in \{1, 2, ... , N\} : ||\mW_n^{(0)}|| \gg \frac{\alpha}{M} ||\sum_{t=1}^{\tau'} {\mDelta_{:,n}^T}^{(t)} \mA^{(t)}||    
\end{equation}
then,
\begin{equation*}
    \frac{\partial \ell}{\partial \va^{(\tau')}} \approx - \vdelta^{(\tau')} \mW^{(0)} = -\sum_{n=1}^{N} \evdelta^{(\tau')}_i \mW^{(0)}_n.
\end{equation*}
It means that the extracted features get credit and penalty for correct and wrong class predictions, but this feedback is weighted by the randomly selected $\mW^{(0)}$.
Concretely, this behavior is not desired and may slow down or even damage the training process. Geometrically, this means that for a long time $\mV_{\mW}$ remains minuscule, so the relative velocity between $\vtheta$ and $\vtheta_{\gT_t}^*$ would not surpass what we calculated in \Eqref{eq:v_theta_and_w_plugged}. 
If the condition stated in \Eqref{eq:initi_gg} persists, the parameters of the feature extractor may eventually adapt so that they can decrease the error. In the extreme case $\mW$ is kept unchanged while updates are only applied to $\vtheta$. 
The empirical outcome of this scenario is shown in Section~\ref{sec:exp_fix_space}.

\subsubsection{How Does FAST Tackle the Initialization Problem?}
FAST maximizes the per-example entropy of the predicted labels by letting
\begin{equation}
    \label{eq:ENTAME_init}
    \forall \ n \in \{1,\dots,N\}, \forall \ e \in \{1,\dots,E\}: ||\emW^{(0)}_{n,e}||^2 = \epsilon \approx 0
\end{equation}
where epsilon is a number close to zero,
defined only in order to maintain logits in the computational graph.\footnote{ 
Usually the practical implementation of softmax is such that maintaining the computational graph is automatically handled even if $\epsilon$ is set to zero. 
}
Unless the magnitude of $\mA$ or the number of features $D$ is extremely large, we should get $\eva^{(0)}_n \approx 0$.
Passing this through softmax function results in  
\begin{equation}
    \label{eq:initial_haty}
    \forall n \in \{1,\dots N\} : \hat{\evy}_n^{(0)} = \frac{1}{N},
\end{equation}
which maximizes the per-example entropy of the predicted labels, that is
\begin{equation*}
    H(\hat{\vy}^{(0)}) \approx -\sum_{n=1}^{N} \frac{1}{N} \ln{\frac{1}{N}} = \ln{N}.
\end{equation*}
From the inference's perspective, the effect of initializing $\mW$ with close to zero values is similar to applying an extremely high temperature to the softmax function.
Defined by~\citet{hinton2015distilling}, softmax with \textit{temperature} $u$ is defined with a slight modification compared to~\Eqref{eq:sm}, such that
\begin{equation}
    p(\vz, u) = \frac{e^{{\vz}/{u}}}{e^{{\vz}/{u}}\vone_{N \times N}}.
\end{equation}

Even if $\epsilon$ is drawn randomly for each entry of $\mW^{(0)}$, \Eqref{eq:initial_haty} still holds with a high precision. Notice that since the logits are close to zero, the exponential functions work in a linear regime and could be approximated by their Taylor series as 
\begin{equation}
e^{\vz^{(0)}} \approx \vz^{(0)} + \vone_{N},
\end{equation}
so the differences in different draws of $\epsilon$ cannot make a big difference in $\hat{\vy}$ as long as their variance is close to zero.
FAST provides $\hat{\vy}$ with the maximum entropy that it can have, since every class is predicted equally per inferred example.

\subsubsection{Initial Reduction in the Variance of the Error}
Maximum entropy of predicted labels can reduce the average variance of the commencing back-propagated error if the first training mini-batch (batch) is balanced.\footnote{In this context, being balanced means having the same number of instances for each class.} To show this, we first expand the variance of the error as in
\begin{equation*}
    {||\vdelta||}_2^2 = {||\vy||}_2^2 + {||\hat{\vy}||}_2^2 - 2 \vy \hat{\vy}^T.
\end{equation*}
Equivalently, the average of this  variance across the instances in the first mini-batch (batch) could be written as  
\begin{equation}
\label{eq:delta_avg_variance}
    \begin{split} \frac{1}{M^{(0)}}\sum_{m=1}^{M^{(0)}} {||\mDelta_m^{(0)}||}_2^2 =
    1 + \frac{1}{M^{(0)}}\sum_{m=1}^{M^{(0)}} {||\hat{\mY}_m^{(0)}||}_2^2 - 2\psi,
    \end{split}
\end{equation}
where the constant value represents the L2-norm of the true labels which is always equal to one, that is ${||\vy||}_2^2=1$. $\psi$ is the average probability assignment for the correct labels or the \textit{soft accuracy} defined as
\begin{equation*}
    \psi= \frac{1}{M^{(0)}}\sum_{m=1}^{M^{(0)}} \hat{\mY}_m^{(0)} {\mY_m^{(0)}}^T.
\end{equation*}
Maximizing the {soft accuracy} is the goal of training. If $M$ is large enough and the first mini-batch (batch) is balanced, the initial soft accuracy is usually a random number with having $\frac{1}{N}$ as its sample mean. On the other hand, using Lemma \ref{th:yhat_norm_sup} or, in a more general way, employing Lemma \ref{th:yhat_norm_sm} along with Axiom \ref{th:sum_max} (all defined in Appendix \ref{appendix:append}), we know that that the average L2-norm of the predicted labels (second term on the right side of equality in \EqrefNP{eq:delta_avg_variance}) is minimized only when the entropy of the predicted labels are maximized for each instance. Hence, considering no prior knowledge about the new task, using our proposed initialization, a notable part of the variance of the error that could potentially mislead SGD toward representations is removed. Specifically, for the case of a balanced first mini-batch (batch), our method yields
\begin{equation*}
    \begin{split} \frac{1}{M^{(0)}}\sum_{m=1}^{M^{(0)}} {||\mDelta_m^{(0)}||}_2^2 =
    1 - \frac{1}{N}.
    \end{split}
\end{equation*}

\subsubsection{Automatic Shift between Training Phases}
In mapping function $f_{dp}$, $\mW$ is the {Jacobian matrix} required in the {gradient chain rule} to back-propagate from the label-space toward the feature-space. Initializing $\mW^{(0)}$ to values that are close to zero makes the back-propagation stall at $\va$ in the first backward phase, i.e.
\begin{equation*}
    \frac{\partial \ell}{\partial \va^{(0)}} \approx \vzero_{D}.
\end{equation*}
However, {the feature-space to label-space projection vectors} are updated such that
\begin{equation*}
    \mW^{(1)}_n = \frac{\lr}{M^{(0)}} {\mDelta_{:,n}^{(0)}}^T \mA^{(0)}
\end{equation*}
which could be rewritten as
\begin{equation}
    \label{eq:w2_separate}
    \mW^{(1)}_n = \frac{\lr}{M^{(0)}} {\mY_{:,n}^{(0)}}^T \mA^{(0)} - \frac{\lr}{M^{(0)}N} \vone^T_M \mA^{(0)}
\end{equation}
It takes a few steps until $||\mW||_F$ and consequently $\frac{\partial \ell}{\partial \va}$ become large enough to notably change the pretrained parameters. We name this period \textit{auto-warmup phase} since it is similar to classifier warmup employed by \citet{li2018learning}, except that instead of manually freezing the feature extractor's parameters, it keeps them almost intact by making $||\mW||_F$ small. Throughout the auto-warmup phase, the parameters of the classifier are modified to form a meaningful connection between the extracted features and the predicted labels. Assuming source and target tasks to have some similarities, this connection can significantly drop the entropy of the predicted labels based on the quality of the updates (that is how much the mini-batches represent the underlying data distribution). Geometrically, the landscape of the target task's loss in $\gS(\vtheta)$ may notably deform through this phase.
After taking enough number of steps, the blocked path to back-propagate the error toward the feature extractor is gradually opened up due to ${||\mW||}_F \neq 0$, and the training can go through its expected behavior.\footnote{Clearly, in this statement we suppose that the features do not come from the output of a layer with all \textit{dead neurons} To get familiar with dead neurons look at \citet{glorot2011deep} and \citet{lu2019dying}.}

\subsubsection[Classifier Parameters as a Feature-class Relevance Matrix]{$\mW$ as a Feature-class Relevance Matrix}
The second term on the right side of~\Eqref{eq:w2_separate} is identical for all rows of $\mW$, so in action, it does not make a predicted label win over another. 
Regardless of such additive value for all label estimators, if $\mY^{(0)}$ is stratified across classes (that is when the first training mini-batch is balanced), the rows of $\mW^{(1)}$ would resemble what \citet{snell2017prototypical} named prototypes. In other words, $\mW^{(1)}$ is a \textit{label-feature relevance matrix} formed based on the information in the first mini-batch (batch) of the extracted features. In the next forward pass of training, we have
\begin{equation}
    \evz^{(1)}_n = \va^{(1)} {\mW_n^{(1)}}^T,
\end{equation}
which implies how similar $\va^{(1)}$ is to the label-feature relevance representation of class $n$, according to what was already learned. Continuing, this 
implicitly forms a sort of moving average over label-feature relevance matrices, calculated from different mini-batches and, as expected, the lower the error becomes, comparably smaller changes are applied.
As we described, initializing $\mW$ to values that are close to zero causes the first updates to make a meaningful link between the obtained features and the objective of the new task. 
However, this deviation from zero could further escalate through the next updates.
Paying attention to \Eqref{eq:recursive_w_fc} reveals that the speed of modifying $\mW$ in each update is directly related to the amount of observed error. Hence, decreasing the training loss would decrease the amount of changes in $\mW$ upon preserving the distribution of extracted features in different mini-batches. Fortunately, due to the prevailing use of normalization layers in deep CNNs, the extracted features are robust to sudden changes even with a large shift in distribution or magnitude of the raw input batches.  

\subsubsection{Scaling up the Number of Classes}
If the number of classes increases, so does the probability of accidentally biasing toward a specific class. Thus, the maximum entropy initialization sounds even more essential. Interestingly, if \Eqref{eq:initial_haty} holds, the average error is not changed by increasing the number of classes. That is because the prediction error is fixed to
\begin{equation*}
    \evdelta_{n}^{(0)} = 
    \begin{cases}
      1-\frac{1}{N} , & \text{if}\ \evy^{(0)}_n = 1 \\
      -\frac{1}{N} , & \text{otherwise}
    \end{cases}.
\end{equation*}
On the other hand, ${||\vy^{(0)}||}_2$ and consequently ${||\vdelta^{(0)}||}_2$ drop by increasing the value of $N$.

\subsection{Progressive Entropy Minimization}
\label{sec:method_lr}
\textit{Overconfidence} is a type of subjective bias by which people's certainty in their judgement is much higher than their objective accuracy~\citep{pallier2002role}. In support of the well-known Dunning–Kruger effect \citep{kruger1999unskilled},  \cite{sanchez2018overconfidence} show that at the beginning of learning a classification task, people gain a lot of overconfidence. By learning more, this surge of confidence usually is followed by a downturn and then a gradual rise that is more correlated with the person's accuracy \citep{sanchez2018overconfidence}. However, at any stage it is ideally desired to match certainty with the amount of knowledge one has obtained about a matter \citep{russo1992managing, winman2004subjective, arkes1988eliminating, speirs2010reducing}.
Moreover, \cite{meyniel2015sense} argue that in humans, confidence is not derived from a heuristic process but rather is formed during their learning process. In the context of classification with DNNs, the probabilities that the logistic softmax function produces are sensitive to the magnitude of the input logits. It makes much more sense to have a gradual decrease in the entropy of the predicted labels while training a DNN; however, when we use transfer learning, a sudden raise in the confidence may happen during the first few updates. FAST prevents the transferred knowledge from being disturbed by a sudden jump in the amount of confidence.

\subsubsection[Controlling the growth of the norm of the classifier parameters]{Controlling the growth of $||\mW||_F$} Forgetting can still happen after the auto-warmup phase. That is, when $||\mW||_F$ grows much faster than a notable amount of decline in the norm of prediction error. In this situation, although $\mW$ is meaningful and back-propagating through it may not mislead gradients on $\vtheta$, it can highly raise their magnitude.  Assuming that a first order GD is used, a very large step, even if pointing toward a minimum in the landscape of the loss, can still be a large damage to the transferred knowledge. Geometrically, in this scenario, at a particular
step the relative speed of $\vtheta$ and $\vtheta_{\gT_t}^*$  in $\gS(\vtheta)$, may become much larger than the distance between them and they keep broadly crossing each other. A more severe case could be observed during a transition to fine-tuning when a classifier warmup phase is used. This is further explained in Section~\ref{sec:exp_wup}. FAST solution to this problem is to control the magnitude with better choices for $\alpha$ and $\beta$.

\subsubsection{The Case of Different Learning Rates}
In Section~\ref{sec:geometric}, we geometrically investigated the impacts of altering $\alpha$ and $\beta$. Initializing $\mW$ with close to zero values makes the role of these learning rates even more pronounced. That is, eliminating ${\mW^{(0)}}^T$ from \Eqref{eq:v_theta_and_w_plugged} makes $V_{\vtheta}$ to be proportional to the product of rates, i.e
\begin{equation}
    \label{eq:v_theta_propto}
    V_{\vtheta} \propto \alpha\beta.
\end{equation}
As indicated earlier, $\alpha$ affects both the speed of $\vtheta$ and the minimum it approaches, and from what we just mentioned this effect is directly proportional to both of them; therefore, $\alpha$ modifies the relative velocity between $\vtheta$ and its closest minimum with a square order. Based on \Eqref{eq:v_theta_propto}, if similar to the traditional fine-tuning we also choose $\alpha=\beta$, the order becomes cubic. 
Therefore, choosing necessarily equal learning rates ties up the velocities of $\vtheta$ and $\vtheta_{\gT_t}^*$ and can cause the anytime performance to be compromised.
The empirical outcomes of separately tuning $\alpha$ and $\beta$ is further expanded in Section~\ref{sec:exp_lr}.

\section{Related works and discussion}
\label{sec:relate}
Initialization and optimization are closely-related subjects in deep learning~\citep{sutskever2013importance}, although they have often been studied individually. In fact, the former may help the latter to speed up the convergence. Optimization algorithms are usually used to search through the space that is formed by the model's parameters. GD, in particular, surfs this space by moving through it and looking at the objective of the search (minimizing the loss) after each step. Undoubtedly, the number of steps to reach the objective is affected by where the first step is taken from. Transferring knowledge in form of pretrained parameters from one neural network to another can be viewed as a kind of parameter initialization. Furthermore, we already have emphasized the  importance of a  customized optimization procedure toward the efficiency of fine-tuning the pretrained models. 
In the following, we take a closer look at how previous works are related to ours in respect of initialization, optimization, and transfer learning concepts.

\subsection{Initialization}
\label{sec:related_init}
Most widely known studies on parameter initialization in DNNs focus on preserving the variance of the flowing data along the depth~\citep{glorot2010understanding, He_2015_ICCV}. This strategy prevents the activations from vanishing/exploding and allows training faster and deeper networks.
Specifically, for a parameterized linear layer with $D$ input and $N$ output neurons, Xavier's initialization method~\citep{glorot2011deep} recommends initializing the parameters from a symmetric uniform distribution with variance equal to $\frac{1}{6D}$ to preserve the variance in the forwarding flow (fan-in mode), and a variance equal to $\frac{1}{6N}$ to preserve the variance in the backwarding flow (fan-out mode). Kaiming initialization~\citep{He_2015_ICCV} takes ReLU activations into account and relaxes the distribution to be symmetric and independent and identically distributed (i.i.d.). They provide variances $\frac{2}{D}$ and $\frac{2}{N}$ for the mentioned directions respectively. \cite{arpit2019benefits} recently acknowledged that the initialization introduced by~\citet{He_2015_ICCV} is the optimal one for a ReLU network trained from scratch. 
They recommended the use of the fan-out mode.
In line with the mentioned initialization methods, to avoid the exploding or vanishing problem, ~\cite{saxe2013exact, mishkin2015all, xiao2018dynamical} initialize the weights with orthogonal matrices. Similarly, \cite{balduzzi2017shattered} proposed \textit{look linear initialization} which prevents gradients in deep feed-forward networks to further shatter along the depth.\footnote{Gradient shattering is defined by~\citet{balduzzi2017shattered} to indicate the phenomenon in which the back-propagating error may further go toward resembling white noise as the depth increases.}

\subsubsection{Initialization for the Classifier}
So far, the research community has paid little attention to the initialization for task-adaptation. In many recent studies, the classifier's parameters that are appended to a pretrained model are initialized using variance preserving methods ~\citep{li2018learning, shermin2018transfer, wang2019easy}. Included in this group, \textit{EfficientNet}~\citep{tan2019efficientnet,xie2019self}, which shows state-of-the-art classification performance on ImageNet's ILSVRC-2012 data set~\citep{russakovsky2015imagenet}, uses classifiers initialized with Xavier's method~\citep{glorot2011deep} for fine-tuning on new tasks~\citep{Tensorflow2017tpu}. Some works do not adjust the initialization based on the embedding size or the number of classes; for instance, \citet{wang2019easy} and \citet{krizhevsky2012imagenet} always choose 0.000025 and 0.001 as the variance of elements of $\mW^{(0}$ respectively.\footnote{\cite{krizhevsky2012imagenet} also use an intercept (bias) vector that is initialized to all ones.}
He et al.~\citep{He_2015_ICCV}, exempted the classification layer of the models used in their experiments from the distribution for initializing weights that they have recommended. This layer's distribution is stated to be found experimentally. Such a strategy could be traced down to even earlier practices in constructing deep neural networks ~\citep{Simonyan2014vgg}. Instead of speculation, we investigate the effect of the initialization of the classifier on the training procedure. However, we narrow down our focus to an optimal initialization for adapting a pretrained model to a new task. Our previous work on classifier initialization~\citep{varno2019efficient} and a recently published paper by~\citet{dodge2020fine} have evidently enlightened the importance of the initialization for task-adaptation. 
FAST uses similar initialization technique as of our previous work~\citep{varno2019efficient} but it highly improves the way its merits are exploited and comprehended.

\subsubsection{Deterministic Initialization for the Classifier} 
Looking into special learning scenarios, two papers \citep{Qi_2018_CVPR, gidaris2018dynamic} concurrently proposed to initialize the rows of $\mW$ used in cosine classifier by averaging the features of the support set for their corresponding class. These average vectors are similar to those~\citet{snell2017prototypical} called \textit{prototypes}. This kind of label-space weight initialization, known as \textit{weight imprinting}, was originally proposed for a \textit{joint-training}.\footnote{We use joint-training to indicate the learning scenario studied by~\citet{Qi_2018_CVPR} according to the terminology used by~\citet{li2018learning}.} The models employed by~\citet{dhillon2019baseline} use a similar technique but only to classify the novel target classes. Our method also initializes the weights deterministically (all zeros) but in contrast to that of~\citet{dhillon2019baseline}, our initialization maximizes the uncertainty at the beginning of the training which leads to smoothly guiding the label-space mapping toward an efficient direction. Classifiers that use FAST promote confidence in results with a gradual and rational pace that is correlated with the knowledge they obtain from the learning task. Additionally, unlike that of~\citet{dhillon2019baseline} our method does not need explicit calculations to find the proper initial values which makes it much easier to apply. 
A method more similar to our initialization is presented by~\cite{zhang2019fixup}---who generalized~\citet{balduzzi2017shattered} for residual networks. 
However, no analytical or even empirical reasoning is provided to support their choice for initializing the classifier's weights.

\subsection{Optimization}
\label{sec:related_opt}
Due to being computationally affordable, first order approximation of GD is the preferred optimization algorithm for training deep learning models. Based on the size of memory required for loading the training data set and the model's parameters, in each step, GD may only be applied on a portion of the data (mini-batch of data). This makes the behaviour of the algorithm dependent on the sequence of the observed mini-batches~\citep{mccoy2020does, dodge2020fine}; thus, the algorithm is often called Stochastic Gradient Descent (SGD) in the context of DNNs. Its simplest form known as \textit{vanilla SGD} applies
\begin{equation*}
    \vtheta^{(t+1)} = \vtheta^{(t)} - \beta \frac{\partial l^{(t)}}{\partial \vtheta^{(t)}},
\end{equation*}
for any set of parameters $\vtheta$.

\subsubsection{Speed of the Convergence}
A line of research focuses on reducing the number of required steps to reach a minimum in the loss landscape \citep{qian1999momentum, botev2017nesterov, duchi2011adaptive, zeiler2012adadelta, tieleman2012lecture, kingma2014adam}. Adding momentum to SGD \citep{qian1999momentum} decreases its dependency on the noise in individual mini-batches. It simply adds a moving average to the gradient of each parameter (look at line 10 of Algorithm \ref{alg:sgd_trad}). From a rather different point of view, \citet{duchi2011adaptive} introduced \textit{AdaGrad} to dampen the oscillation along very steep dimensions in the loss landscape.\footnote{Steepest dimension is the dimension that determines the degree of \textit{Lipschitz continuity} and could vary from one neighborhood to another in the loss landscape.}
It normalizes the gradient of the loss with respect to each parameter with its L$_2$-norm that is summed over all the previous steps. It could be thought as modifying the vanilla SGD by scaling the learning rate of each parameter $\evtheta_i$ at any step $t$ by $g^{(t)}_{\evtheta_i}$ defined as
\begin{equation}
    \label{eq:adagrad}
    g^{(t)}_{\evtheta_i} = g^{(t-1)}_{\evtheta_i} + \left(\frac{\partial \ell^{(t)}}{\partial \evtheta_i^{(t-1)}}\right)^2.
\end{equation}
Since elements in this summation are always positive, the learning rate keeps scaling down as the training goes on until when the optimization algorithm becomes unable to modify any parameter. To address this issue, Adadelta~\citep{zeiler2012adadelta} and RMSProp~\citep{tieleman2012lecture} were introduced with a slight difference. They altered \Eqref{eq:adagrad} such that instead of taking all of the previous optimization steps into account only a weighted average over a window of the steps is considered. 
Taking advantage of a hybrid strategy, Adam~\citep{kingma2014adam} integrated momentum with RMSProp and is known for notably speeding-up the convergence compared to SGD that only works with with momentum.

\subsubsection{Performance Gap}
Despite its merits, Adam has been largely criticized for asymptotically being outperformed by SGD with momentum~\citep{keskar2017improving, Reddi2018convergence}.
Dampening the gradients along the dimensions with a high frequency  \citep{zeiler2012adadelta, tieleman2012lecture, kingma2014adam} is found to gain over SGD only in the beginning of the training~\citep{keskar2017improving} which corresponds to moving across regions of the loss landscape with high altitudes. These regions are often more chaotic compared to regions closer to the minima~\citep{li2018visualizing}. Motivated by these facts, we will empirically show that in the case of fine-tuning on a similar task, the gap between the performance of Adam and SGD with momentum could be significantly reduced. This mainly comes from the fact that, unlike the traditional fine-tuning, optimizing the classification objective with FAST does not make parameters to step out of the proximity of $\vtheta_{\gT_2}^*$ (the minimum in the loss landscape of the target task) far onto the high altitudes.

\subsubsection{Initial Variance of the Error}
\cite{vaswani2017attention} and \cite{popel2018training} showed that a warmup phase for adaptive learning rate can significantly accelerate the training convergence.\footnote{This is different from the classifier warmup phase in which only the classifier is updated.} It is an initial training phase in which a very small learning rate is applied. \cite{liu2019variance} found the source of this phenomenon in the large initial variance of the gradients. They demonstrate this by showing that the distribution of the gradients has momentous changes in the first few optimization steps and addressed the issue by an optimization algorithm that dampens the variance accordingly. Similarly, \cite{luo2019adaptive} and \cite{zhang2019adam} proposed to set dynamic boundaries on the magnitude of the updates. \cite{keskar2017improving} proposed to start training with Adam to take advantage of the convergence speed and then switch to SGD with momentum to better generalize when converged. Our proposed optimization algorithm for fine-tuning also decreases the initial variance of the gradients and similarly accelerates the training convergence. 
Using FAST, neither the speed of convergence nor the generalizable performance at convergence are scarified. 
FAST finds its merits in controlling the velocity of applying updates
only via the first gate that gradients back-propagate through (e.g., the classifier layer). Although our work focuses only on classification, the analysis we provide suggests that in an already stable model (e.g., pretrained on a similar task, so at comparably low altitude regions of the loss landscape), SGD with momentum is able to show a competent convergence speed as long as the pretrained parameters are not disrupted by sudden and large-variance gradients, which can cause overshooting the minimum of the loss landscape that they aim to converge. It is worth mentioning that label smoothing~\citep{Szegedy_2016} is also relevant in the sense that it reduces the initial variance of the gradients by decreasing the L$_2$-norm of the true labels, while our method focuses on the norm of the predicted labels for which the magnitude could be automatically adapted during the course of training (unlike the norm of the true labels).

\subsubsection{Other Related Works on Optimization}
\citet{loshchilov2016sgdr} aim to improve the anytime performance while training DNN models
with a scheduling the value of the learning rate. Although using \textit{learning rate schedulers} have been shown to be a promising direction, we choose not to experiment with it to avoid introducing conflating factors into our study. 
Moreover, depicting the loss landscape in the other spaces rather than the space of all parameter have not been a common strategy in the deep learning community. Among the particular works that have employed this strategy,~\citet{yu2019interpreting} found visualizing the loss landscape in the input space to be more inline with their findings about \textit{adversarial vulnerability and robustness} of neural networks.
Furthermore, although our analysis 
about the connection between the minima in loss landscapes of two similar tasks is novel, depicting such connection between minima in the loss landscape of a single task has been already investigated. In this line of research, ~\citet{becker2020geometry} found the depth (number of layers) to be an indicator of how strided the minima are spread on the loss surface. ~\citet{draxler2018essentially} disputed the common perception about the location of minima of the loss landscape of CNN models at the bottom of distinct valleys and suggested that they are connected in a common low-altitude valley.

\subsection{Transfer Learning}
\label{sec:related_tl}
Although the literature is not unanimous about the exact concept for which \textit{transfer learning} should refer to, in its most widely known meaning it indicates borrowing knowledge from a learned task to learn a new one \citep{pan2009survey}. It usually is categorized into \textit{feature extraction} and \textit{fine-tuning}~\citep{li2018learning}, where the former means only the appended parameters are altered during the training on the target task, while the latter also considers adapting the pretrained parameters.
It is well-accepted among the community that fine-tuning could accelerate the training procedure compared to learning \textit{from scratch};\footnote{Learning \textit{from scratch} refers to a setting in which the parameters of the model are randomly initialized. In deep learning literature, learning from scratch sometimes, comes as the opposite case of transfer learning.} however, whether it also improves the outcomes~\citep{Kornblith_2019_CVPR, raghu2019transfusion} or they results asymptotically emerge~\citep{He_2019_ICCV} is still an open debate. More recently, the latter hypothesis got back the attention of the community since the models that exhibit state-of-the-art classification performance on many machine vision benchmarks are those that are pretrained on larger source data sets~\citep{wang2019easy, xie2019self}.

\subsubsection{Catastrophic Forgetting} 
Fine-tuning a pretrained neural network on a new task, greatly degrades the performance of the model on the source task. This phenomenon is widely known as \textit{catastrophic forgetting}  \citep{mccloskey1989catastrophic}. \cite{goodfellow2013empirical} empirically found that dropout \citep{hinton2014dropout} reduces catastrophic forgetting although at the cost of performing sub-optimal on the new task. 
\citet{kirkpatrick2017overcoming} introduced a way to spot the parameters that are more important for the source task and to assign a smaller learning rate to them. 
The practical importance of less forgetting is highlighted in a generic learning scenario so called \textit{continual learning}~\citep{ven2019scenarios, hsu2018re}. It consists of several learning sub-scenarios with real life applications that take into account incrementally learning new tasks~\citep{li2018learning}, domains~\citep{ven2019scenarios} or classes~\citep{rebuffi2017icarl}. Due to its practicality, continual learning has recently become more popular and is even studied combined with other learning scenarios~\citep{lao2020continuous}. We consider exploiting the knowledge preserving aspect of FAST in continual learning as a possible direction for the future works. However, in this paper we retain our focus on the classic transfer learning where interestingly, in contrast to the previous studies, our method does not
make an explicit effort to retain the performance of the source task and does not compromise the performance on the target task.
Our main goal is to accelerate the fine-tuning process on the target task without compromising its convergence performance and as a bi-product the source task is also forgotten less than in a typical fine-tuning procedure.

\subsubsection{Meta-learning}
\textit{Meta-learning} or \textit{learning to learn}~\citep{schmidhuber1993neural} is a generalization of the classic transfer learning, where the transferred knowledge is obtained from a collection of tasks instead of only one. It has been a while since meta-learning has become the leading research direction to tackle few-shot learning \citep{finn2017model, nichol2018first, li2017meta, mishra2018a}. However, recently, it has been shown that state-of-the-art meta-learning methods are not much superior compared to a more complex fine-tuning \citep{dhillon2019baseline} or even feature extraction \citep{chen2019a_closer}.

\subsubsection{Classifier Warmup}
\citet{li2018learning} considered a classification warmup phase before fine-tuning on a new target task. In this phase, the feature extractor works in its inference mode and only the parameters of the classifier are updated. This is similar to finding a better initialization for the classifier though not only the representation learning is postponed, the number of required updates in the classification warmup phase is not known. The optimal number depends on many factors, including the values drawn for $\mW^{(0)}$. \citet{li2018learning} used a validation set to determine when the classifier warmup phase should be ended. Our experiments suggest that, the transition from feature extraction to fine-tuning in this form causes a minimum overshooting, which immediately damages the model's performance (see Section~\ref{sec:exp_wup} for a more complete explanation). In contrast, our method does not demand manually switch between training phases. Instead, it chooses the initialization and the learning rate of $\mW$ such that the magnitude of the updates on the feature extractor's parameters become proportional to both the error and its certainty (through gradual entropy minimization).

\section{Experiments}
\label{sec:exp}
In this section, we first define our methodology to compare anytime performance of different models and then use it to quantitatively validate our hypothesis through different experiments. 

\subsection{Evaluation Methodology}
To evaluate different models or training procedures applied on a standard supervised task, three labeled partitions of data are commonly used. \textit{Training set} helps GD to move through the parameters' space. \textit{Validation set} is mainly used to monitor the generalizable performance during training. If the inferred performance (or another measure like loss) on validation set improves, the state of the model is stored, so it could be deployed when the training ends without any further progress. \textit{Test set} is used to compare the performance of different models deployed after training has ended. In this scheme, validation set compensates for the non-monotonic trend of outcomes that GD may show during training; however, the speed of achieving a certain level of generalizable performance is disregarded and the judgement is made only based on the final test results. 
To better justify training procedures that guide the model through trajectories in the parameters' space with better outcomes during the course of training, the deployment results could simply be monitored at a pre-defined set of deployment checkpoints. This sounds more inline with the train-deploy cycle through which a human being usually learns to perform a single task.

There have been a lot of research works on speeding-up the convergence of neural networks, yet only a few have shown their gain in a quantitative way (e.g., ~\citet{raghu2019transfusion}). Some DNN classifiers have been viewed and analyzed in the context of anytime algorithms but often in terms of~\textit{anytime prediction} which is concerned about the inference time \citep{hu2019learning, huang2017multiscale, Amthor_2016, wan2019alert}---that is, the performance at inference time could be varied based on the time and energy constraints. In contrast, we introduce a new quantitative way to show how much the learning process of a model is sped-up at certain deployment checkpoints, and exploit it to compare our proposed methods with the baselines. \textit{Anytime Deployment Performance (ADP)} is the term we use to indicate the results that are collected in this way.
ADP reflect the stability of the generalizable results in terms of an anytime algorithm, which tolerates the interruption of the training at almost anytime. 
We define ADP simply to be the test performance on the best saved model until certain checkpoints which are pre-determined by a sequence of the number of training steps.
For instance, let us assume that during the training, the state of the model is saved exactly after $t$ training updates because evaluating it on a validation set at that step shows the best validation performance up to $t$. If the smallest deployment checkpoint that is equal or larger than $t$ is $\tau$, and from $t$ to $\tau$ the validation performance does not improve, then the state of the model that was stored right after the $t$-th update is retrieved at $\tau$ and its performance is measured on the test set. These measures are reported as the ADP. 

\subsection{Effect of Classifier Initialization}
\label{sec:exp_init}
In this experiment, we aim to validate our hypothesis about the influence of carelessly initializing the parameters appended through model adoption using the concept of catastrophic forgetting. Let $\gT_s$ and $\gT_t$ be the source and target tasks respectively; that is, the model is already trained on $\gT_s$ and the goal is to fine-tune it on $\gT_t$. We want to show that for currently well-known fine-tuning techniques, at the beginning of the training, the performance of the model on $\gT_s$ suddenly drops without much gain in its performance on $\gT_t$. Additionally, we want to see how forgetting to perform $\gT_s$ would impact the speed of convergence on $\gT_t$.

\subsubsection{Setup}
To cover engagement of features in different levels of representations, we consider two scenarios: first, choosing $\gT_s$ and $\gT_t$ to be identical and second, choosing them to be unrelated and from completely distinct domains.
A \textsc{ResNet-18}~\citep{He_2016_resnet} that is already pretrained on ImageNet's \textsc{ILSVRC-12}~\citep{russakovsky2015imagenet} is used for this experiment. Five percent of the training split of each task is separated for validation in a stratified order. Training, validation, and test images are resized to 128$\times$128. Random horizontal flip is applied on images of the training mini-batches. Size of each training mini-batch is 100 ($M=100$) with equal number of images per label. SGD is used with learning rate of 0.01 and gradient momentum equal to 0.9. 
The number of gradient updates (steps) between two consecutive validations is increased gradually starting from 1 and saturating at 10. ADP checkpoints are set to the update numbers corresponding to the \textit{Fibonacci sequence}, except for the third number which is skipped because it is equal to the second number in the sequence.

The source task is chosen to be classification defined over \textsc{CIFAR-100} data set \citep{krizhevsky2009learning}. \textsc{CIFAR-100} includes 32$\times$32 labeled RGB images of natural objects from 100 classes with 500 training and 100 test examples per class. $\mW$ is initialized according to the fan-in mode of Kaiming's method \citep{He_2015_ICCV} and trained for 120 epochs.
In the first scenario, the target task is chosen to be identical to the source task, and as usually is done in fine-tuning, in the transition from $\gT_s$ to $\gT_t$ the pretrained classifier is replaced with a random one. The second scenario, chooses $\gT_t$ to be the classification defined over \textsc{MNIST} data set~\citep{lecun1998mnist}. \textsc{MNIST} consists of 28$\times$28 gray-level labeled images of handwritten digits with 60000 training and 10000 test examples. Multiple sets of models are fine-tuned, initializing $\mW$ by drawing from zero-centered normal distributions with different standard deviations ranging from 1 down to $10^{-8}$ for each set of models.
Normal distribution is shown with $\mathcal{N}$ which takes its mean and standard deviation as its first and second arguments respectively. For example, to show that $\mW$ is initialized with zero-centered normal distribution that has a standard deviation equal to 1, we note $\pmb{W}^{(0)}\sim\mathcal{N}(0,1)$.
Each experiment is repeated using a set of 15 different random seeds. In our setup, applying different random seeds not only initializes $\mW$ differently but also exposes dissimilar sequence and combination of training mini-batches to the model.\footnote{To study further about the influence of initialization and the order of mini-batches see \citet{mccoy2020does}.}

During fine-tuning, the model is deployed on both the source and target tasks at determined checkpoints (pseudo-Fibonacci sequence). To deploy the model on the source task in this stage, the test split of the source task is first fed to the feature extractor in its inference mode, then it is passed through the classifier that was trained on the source task and has been detached through the model adoption process.

\subsubsection{Results and discussion}
Scaling down the magnitude of the initial values of $\mW$ not only increases the speed of convergence to the target task, but also retains more knowledge from the source task. This is concluded from the ADP performance of these tasks shown in Figure \ref{fig:init_vs_forget_c100}. Although the rising $||\mW||_F$ in some curves shown in Figure \ref{fig:init_vs_forget_c100_w} indicates that $\vtheta$ is considerably modified (unlike feature extraction where $\beta=0$) after just a few steps, the corresponding deployment performance of the source task plotted in the second row of Figure \ref{fig:init_vs_forget_c100} suggests that almost no forgetting happens when $\gT_s = \gT_t$ and $\mW^{(0)} \longrightarrow 0$. 

\begin{figure}[t]
\centering
    \def\svgwidth{15cm}
    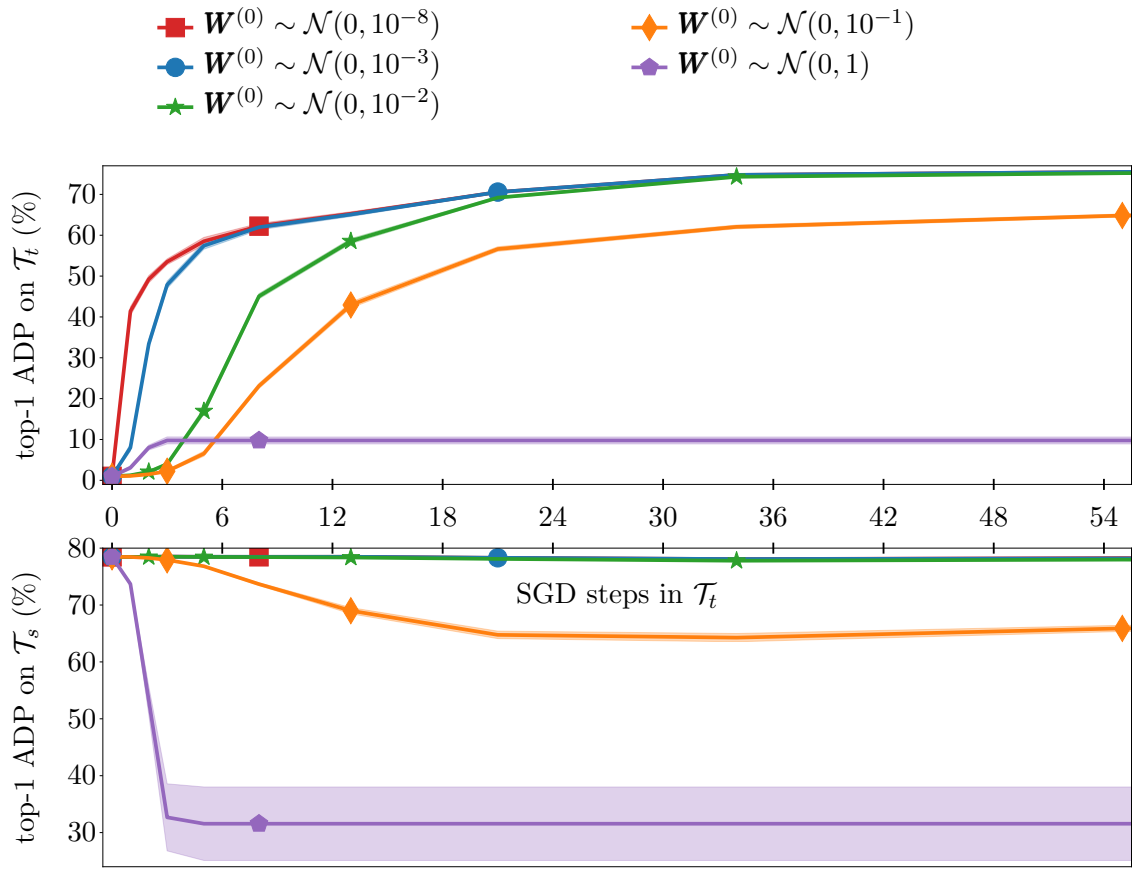
    \caption[The effect of initializing the classifier parameters]{Understanding the effect of initializing $\mW$ through catastrophic forgetting. 
    ADP accuracy on $\gT_t$ at the top and $\gT_s$ at the bottom. Both tasks are identical and defined by classification on \textsc{CIFAR-100} data set.
    The horizontal axis shows the optimization steps on the target task which is shared among the two plots.
    }
    \label{fig:init_vs_forget_c100}
\end{figure}

\begin{figure}[t]
    \centering
    \def\svgwidth{14.6cm}
    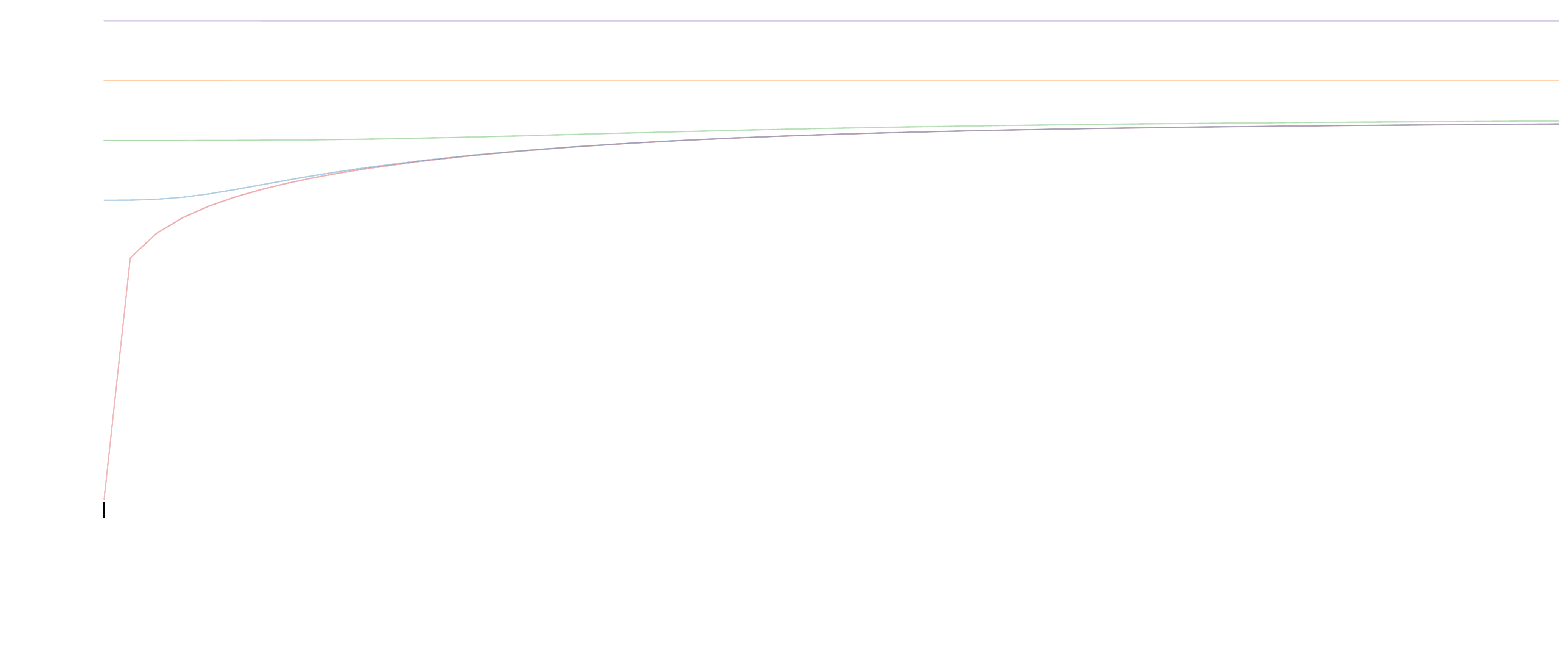
    \caption{The development of $log(||W||_F)$ corresponding to fine-tuning on $\gT_t$ shown in Figure \ref{fig:init_vs_forget_c100}.}
    \label{fig:init_vs_forget_c100_w}
\end{figure}

Figure \ref{fig:init_vs_forget_mnist} outlines the results of the second scenario ($\gT_s \neq \gT_t$). Notice that according to these results, when $\mW^{(0)}$ is closer to zero, within only 50 SGD steps (visiting total 5000 input images), ADP of $\gT_t$ reaches over 96\% accuracy while showing 20\% ADP accuracy on $\gT_s$. This is notable, considering the target task being in a totally different domain and the source task having 100 classes (compared to the expected accuracy from random outputs for a balanced task at 1\%). 

Rooted from the adverse effect of retaining the condition stated in \Eqref{eq:initi_gg} for a long time and its connection with curves in Figure \ref{fig:init_vs_forget_c100_w} and their corresponding ADP results in Figure \ref{fig:init_vs_forget_c100}, we conclude that \textbf{initializing the classifier appended through model adoption with close-to-zero values, not only better preserves the transferred knowledge but also can accelerate the training procedure on the target task}.

\begin{figure}[t]
\centering
    \def\svgwidth{15.0cm}
    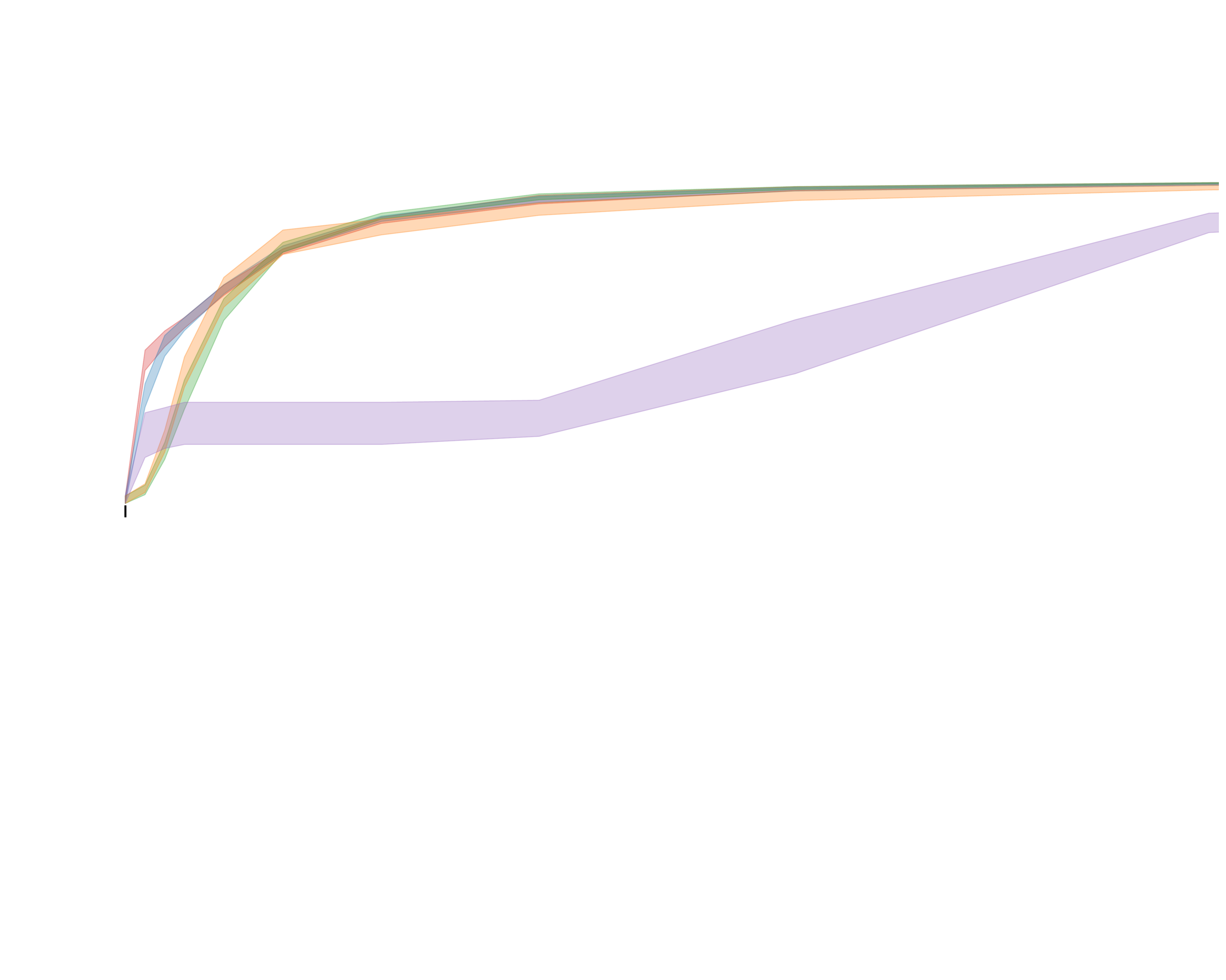
    \caption{Understanding the effect of initializing $\mW$ through catastrophic forgetting. ADP accuracy on $\gT_t$: \textsc{MNIST} at the top and $\gT_s$: \textsc{CIFAR-100} at the bottom.
    The horizontal axis shows the optimization steps on the target task which is shared.}
    \label{fig:init_vs_forget_mnist}
\end{figure}

\subsection{Effect of Parameter Heterogeneity}
\label{sec:exp_fix_space}
It is already explained analytically (in Section \ref{sec:method_init}) and empirically (in Section \ref{sec:exp_init}) that why initializing $\mW$ with large and random values, misleads the optimization algorithm to take large and aimless steps in $\gS(\vtheta)$ at the beginning of the fine-tuning. Also we showed that these steps cause the model to  forget the transferred knowledge and consequently slow down the convergence. One interesting question is how much of this damage comes from the randomness of $\mW$ apart from its magnitude.
We try to answer this question in this experiment and partly in the next one.

Geometrically, for the period that \Eqref{eq:initi_gg} holds, the landscape of the target task's loss does not deform much in $\gS(\vtheta)$ (the target minimum in Figure \ref{fig:space_classic} stays almost hold for that period). In this situation, if $\mW^{(0)}$ is carelessly drawn, it is very likely that the trajectory through which SGD guides $\vtheta$, becomes long and hard to settle. To further pronounce this phenomenon, 
by setting $\alpha$ to zero, we force the optimization algorithm to keep $\mW$ as it is initialized; hence, the convergence only relies on $\vtheta$ traveling in $\gS(\vtheta,\ell)$. This makes us able to abstractly perceive the effect of the fast geometric deformation that FAST provides.

\subsubsection{Setup} We compare ADP when randomly-initialized $\mW$ is kept unchanged ($\alpha=0$) during fine-tuning compared to when it is updated with the same rate as $\vtheta$ ($\alpha=\beta$). For both cases we choose $\mW^{(0)} \sim \mathcal{N}(0, 10^{-2})$. The rest of the setup is identical to the first scenario of the previous experiment except that fine-tuning is done on \textsc{CIFR100} as the target task ($\gT_{2}$) by adopting an off-the-shelf model pretrained on ImageNet's \textsc{ILSVRC-12}~\citep{russakovsky2015imagenet} ($\gT_{1}$).

\subsubsection{Results and discussion}
Figure \ref{fig:hetero} suggests that when the loss landscape deforms as discussed in Section~\ref{sec:geometric}, the convergence is significantly accelerated. The shadows representing 95\% confidence intervals are barely visible indicating that the results shown in this figure are highly consistent. The results in this experiment strongly support our hypothesis which recognizes \textbf{the deformation of the loss landscape in the space of the transferred parameters to largely influence the convergence speed}.

\begin{figure}[t]
\centering
    \def\svgwidth{15.0cm}
        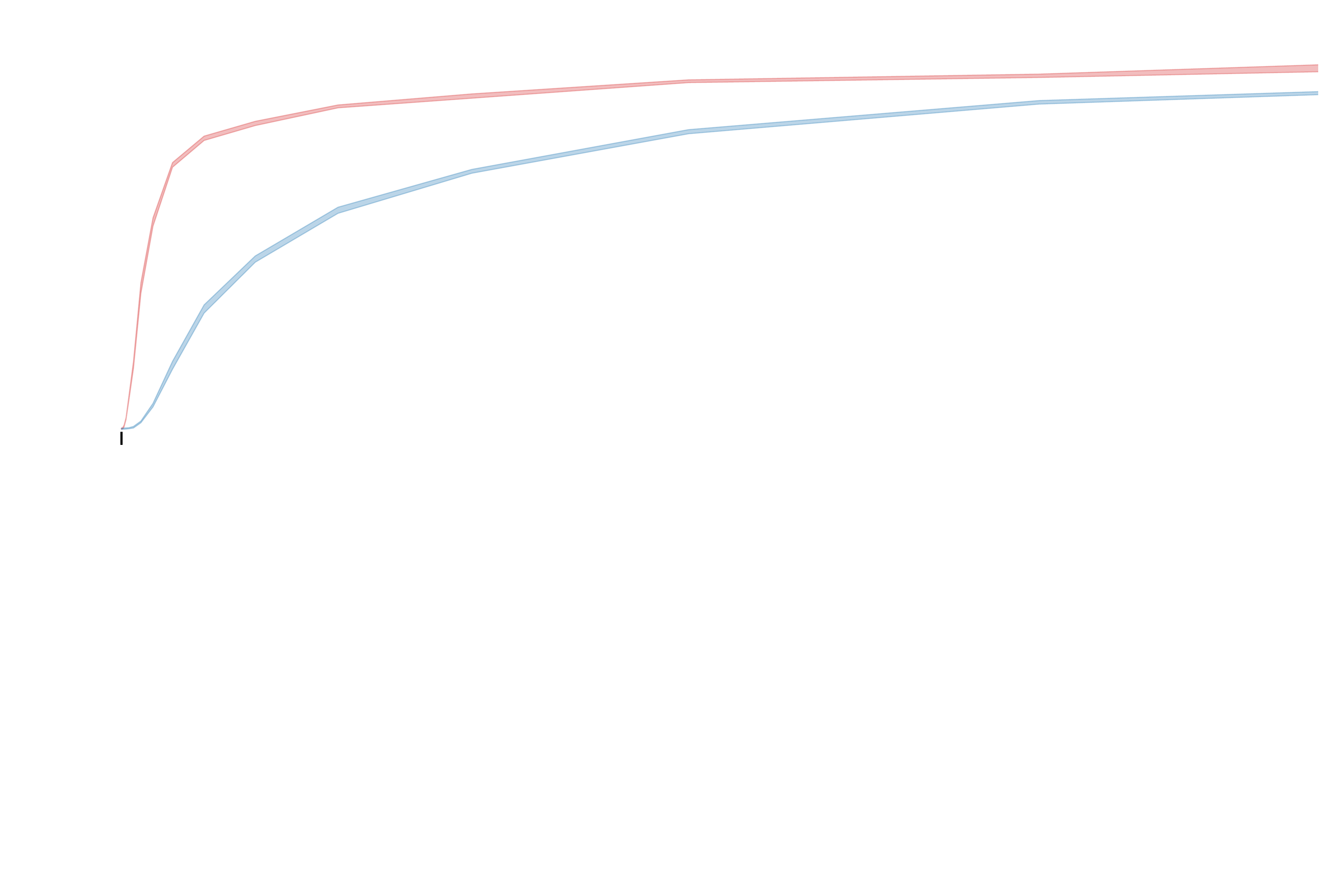
    \caption{The effect of preventing the loss landscape to deform in $\gS(\vtheta)$ during fine-tuning.}
    \label{fig:hetero}
\end{figure}

\subsection{Classifier Warmup}
\label{sec:exp_wup}
Introduced by \citet{li2018learning}, fine-tuning could be done after a classifier warmup phase during which only the classifier is modified. The transition from this phase to fine-tuning is committed whenever the validation performance does not improve for a number of training steps. In this experiment we show how using classifier warmup effects the generalizable performance while training.

\subsubsection{Setup}
The warmup phase is stopped whenever the validation accuracy has not improved at least one percent for 10 consecutive SGD steps. The rest of the setup is akin to that of Section~\ref{sec:exp_fix_space}.

\subsubsection{Results and discussion}
Large leaps of validation accuracy pointed with arrows in the first row of Figure \ref{fig:exp_wup} show the step at which the transition from classifier warmup phase to jointly training $\vtheta$ and $\mW$ is committed. At these points, the back-propagate gradients have comparably large magnitude which makes SGD take a big (but unnecessary) step in $\gS(\vtheta)$. The size of this step depends on $||\mW||_F$ which in turn depends on $\mW^{(0)}$ and $\alpha$. In this experiment, we use $\alpha=10^{-2}$ but initializing the classifier in different ways. The orange and green curves correspond to the fan-out and fan-in modes of Kaiming initialization \citet{He_2015_ICCV} respectively. The blue curves shows the effect of letting the running statistics---used in normalization layers of feature extractor---to be updated during the training of classifier. Similarly, the running statistics are updated for the case shown with red curves but the classifier is initialized with small-variance values in this case. The reason for including cases with updating running statistics is to opt-out its influence factor and make more certain conclusions.

Unlike other experiments, the visualization made for this experiment only reflects a single random seed (no shadow is plotted to indicate confidence interval) and the outcomes of runs with other random seeds are not shown so to let the minimum overshooting leaps be visually clear. The optimization steps where sudden leaps in the performance of different curves takes place in Figure \ref{fig:exp_wup}, support our geometric hypothesis visualized in Figure \ref{fig:space_wup}. In summary, the results presented in this experiment suggest that \textbf{although including a classifier warmup phase can help the optimization algorithm to take its initial steps in a correct direction, it still is prone to minimum overshooting.} In the sever case of minimum overshooting, $\vtheta$ may step into high-altitude regions of the loss landscape which makes the convergence non-efficient and can cause catastrophic forgetting. 

\begin{figure}[t]
    \centering
    \def\svgwidth{15.0cm}
    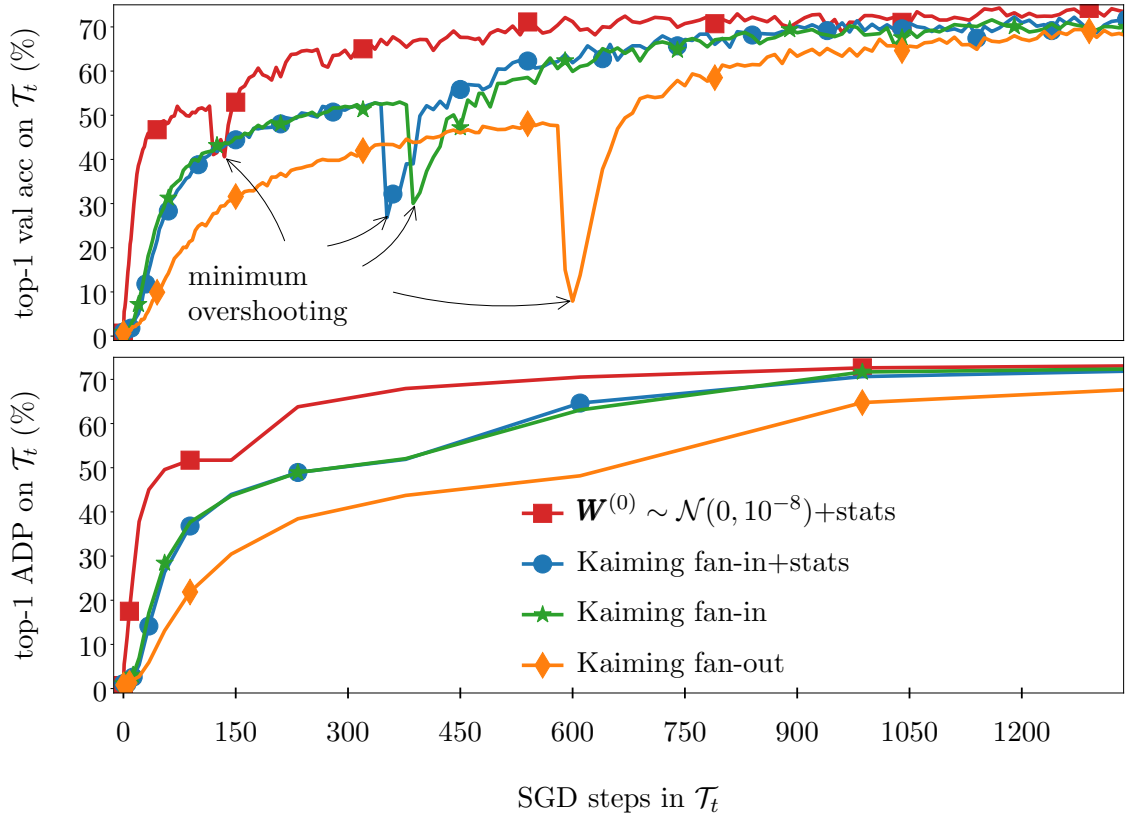
    \caption{Classifier warmup is still prone to overshooting the minimum. The plots show the validation performance and ADP from top to bottom.}
    \label{fig:exp_wup}
\end{figure}

\subsection{The Effect of Unequal Learning Rates}
\label{sec:exp_lr}

As shown in Section~\ref{sec:method_lr}, the effect of $\alpha$ on the relative velocity of $\vtheta$ and $\vtheta_{\gT_t}^*$ is much larger than that of $\beta$. This implies that for an efficient convergence rate, the learning rates of the feature extractor and the classifier should be adjusted independently or at least should not necessarily be set equally as is traditionally done. Considering~\Eqref{eq:recursive_w_fc}, this is even more pronounced when $\mW$ is initialized with close to zero values.  
In this experiment we compare the outcome of scaling the learning rates in symmetric and asymmetric ways. To do so, starting from a baseline, we scale-up and scale-down $\alpha$ and $\beta$ both equally and unequally, and inspect the outcomes.

\subsubsection{Setup}
In this experiment we use \textsc{ResNet-18}~\citep{He_2016_resnet} and \textsc{VGG-19}~\citep{Simonyan2014vgg} which are pretrained on ImageNet's \textsc{ILSVRC-12}~\citep{russakovsky2015imagenet} to fine-tune on the classification task defined over \textsc{CIFAR-100} data set~\citep{krizhevsky2009learning}. 
For all the cases, $\mW$ is initialized randomly from a zero-centered normal distribution with a close to zero standard deviation ($10^{-8}$). The rest of the setup is akin to that of Section~\ref{sec:exp_fix_space} except that the learning rates for SGD optimization are explicitly expressed in plots on each curve.

\subsubsection{Results and discussion}
Figure~\ref{fig:lr_sym_asym} compares the top-1 ADP accuracy (top row) and the norm of $\mW$ (bottom row) when the learning rates are scaled in a symmetric way (Figure~\ref{fig:lr_res18_sym}) versus in an asymmetric way (Figure~\ref{fig:lr_res18_ass}) for the \textsc{ResNet-18}~\citep{He_2016_resnet}. The blue curves are identical in this Figure and are just repeated so can easily be compared to the other curves. Symmetrically scaling down the learning rates shown by the red curve on Figure~\ref{fig:lr_res18_sym} improves the ADP eventually, though it degrades the performance in the first 100 optimization steps. On the other hand, as shown by the red curve on Figure~\ref{fig:lr_res18_ass}, if only $\beta$ is scaled down, the best performance is similarly improved but is not compromised on the earlier steps at all. 

For \textsc{VGG-19}~\citep{Simonyan2014vgg} which compared to \textsc{ResNet-18}~\citep{He_2016_resnet} is a more challenging model to train~\citep{li2018visualizing}, the difference between the aforementioned learning-rate scaling strategies seems to be bolder. This is depicted in Figure~\ref{fig:lr_vgg19} where the asymmetric scaling more significantly speeds-up the ADP performance. 

The results in this experiment indicate that in order to have an efficient ADP,~\textbf{the learning rates of the pretrained feature extractor and appended classifier are better to be tuned separately}. This statement does not intend to convey that, these learning rates do not influence the optimal value of each other but rather it emphasizes that the prevailing practice to set them equal could compromise the performance either at the beginning of the training or for a long time afterward. Furthermore, the compromise seems to be more significant as much as the loss landscape is more chaotic (e.g., \textsc{VGG-19} compared to \textsc{ResNet-18} as explained by~\citet{li2018visualizing}).

\begin{figure}[t]
\centering
    \begin{subfigure}{0.52\textwidth}
        \centering \def\svgwidth{7.9cm} 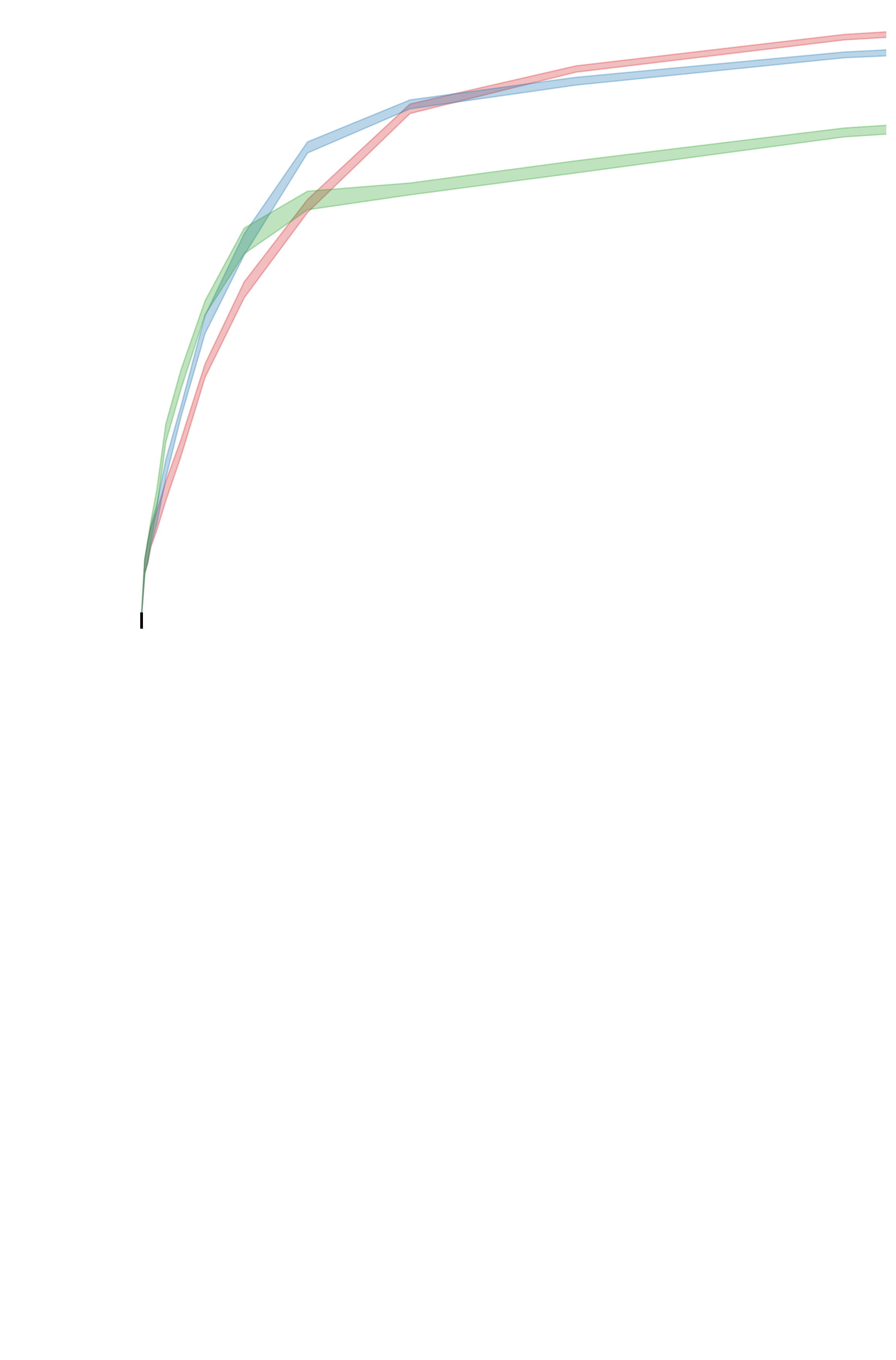
        \caption{Symmetric scaling}
        \label{fig:lr_res18_sym}
    \end{subfigure}
    \begin{subfigure}{0.47\textwidth}
        \centering
        \def\svgwidth{7.2cm}
        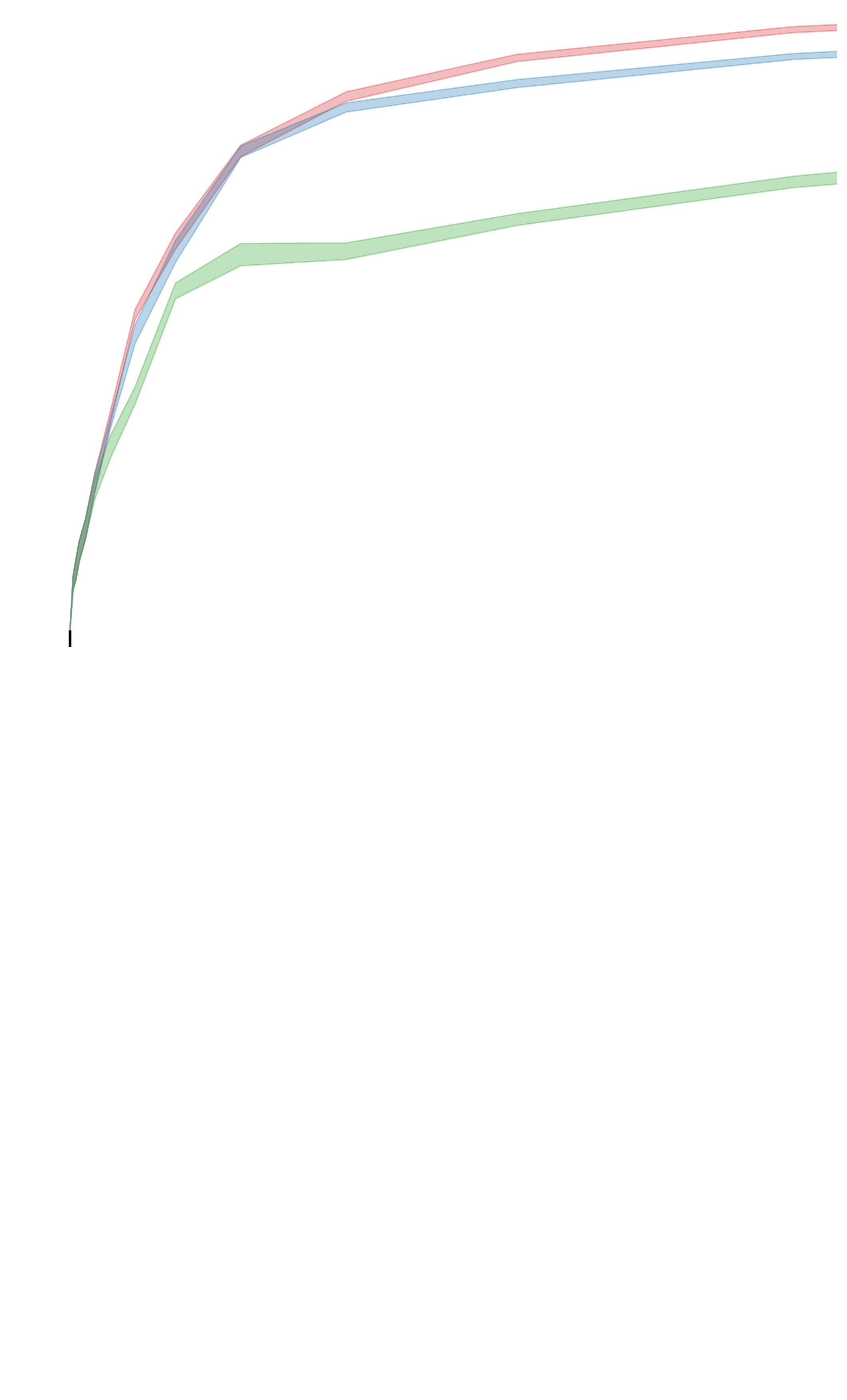
        \caption{Asymmetric scaling}
        \label{fig:lr_res18_ass}
    \end{subfigure}
    \caption{Comparing the effect of scaling the learning rate in (a) symmetrically versus (b) asymmetrically for fine-tuning on \textsc{ResNet-18}~\citep{He_2016_resnet}. In all cases $c=10^{-2}$.}
    \label{fig:lr_sym_asym}
\end{figure}

\begin{figure}[t]
    \centering
    \def\svgwidth{15.0cm}
    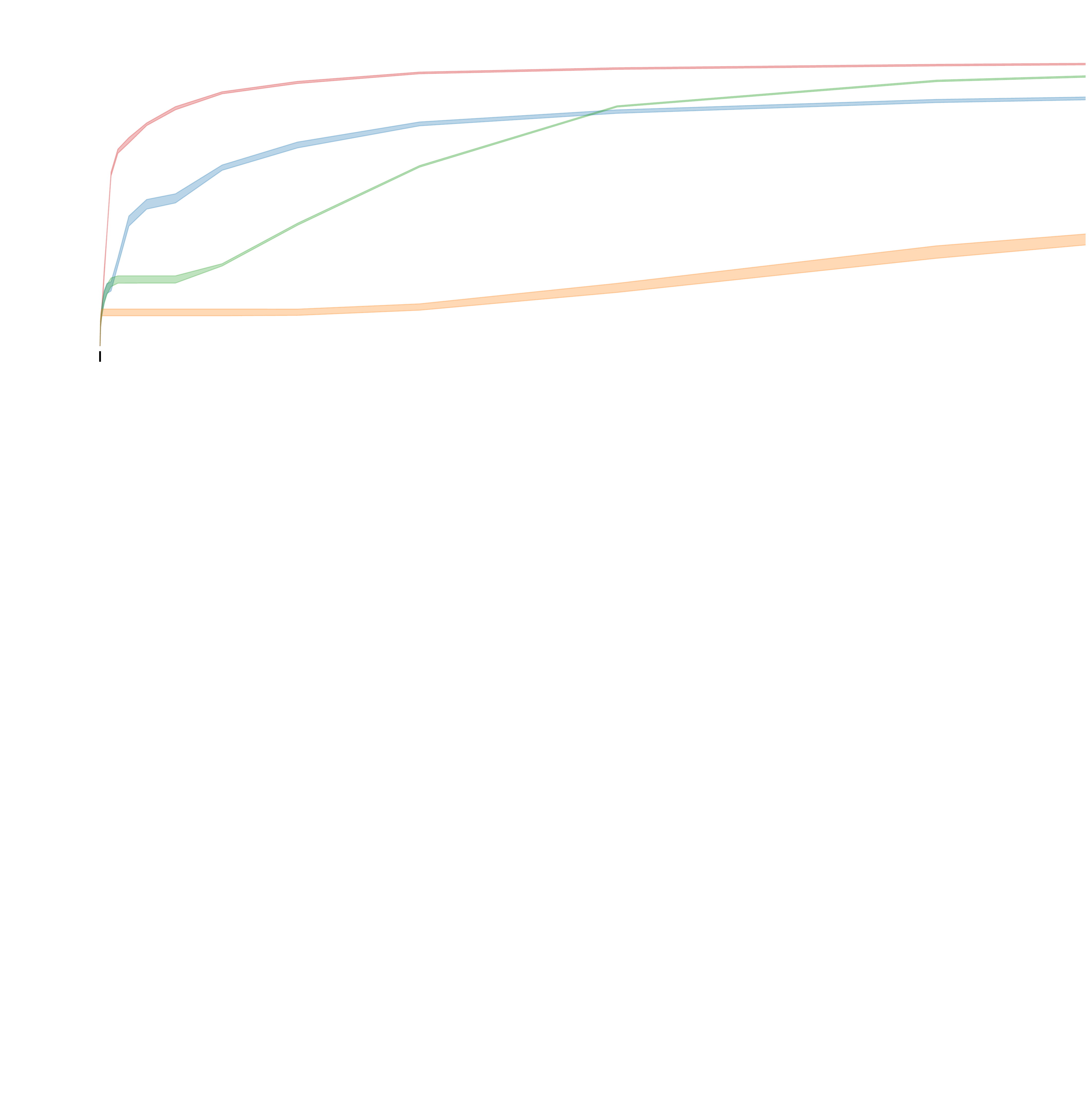
    \caption{Comparing the effect of scaling the learning rate in (a) symmetrically versus (b) asymmetrically for fine-tuning on \textsc{VGG-19}~\citep{Simonyan2014vgg}. In all the cases $c$ is equal to $0.05$ and $\mW^{(0)} \sim \mathcal{N} (0, 10^{-8})$.}
    \label{fig:lr_vgg19}
\end{figure}

\subsection{Optimization}
\label{sec:exp_opt}
Putting the outcomes of Sections~\ref{sec:exp_init} and~\ref{sec:exp_lr}, we want to compare the complete FAST method with a fine-tuning baseline. However, we first need to answer two fundamental questions. 
\begin{itemize}
    \item First, {are the outcomes of employing FAST essentially different from those of the baseline} or similar convergence behavior could be obtained by tuning the learning in the baseline?
    \item Second, {does FAST reduce the performance gap between SGD with momentum and Adam as hypothesized} (see Section~\ref{sec:related_opt})?
\end{itemize}
In this experiment we look for answers to these question from an empirical point of view.

\subsubsection{Setup}
In this experiment we use~\textsc{VGG-19}~\citep{Simonyan2014vgg} as our challenging optimization case in addition to \textsc{DenseNet-201}~\citep{Huang_2017} which according to what~\citet{li2018visualizing} states about the skip connections, is characterized with a relatively smooth loss landscape despite having a lot of layers. Both models are pretrained on ImageNet's~\textsc{ILSVRC-12}~\citep{russakovsky2015imagenet} and the goal is to fine-tune them on the classification task defined over~\textsc{CIFAR-100} data set~\citep{krizhevsky2009learning}. The learning rates and initialization setup for FAST and the baselines are explicitly stated on the plots. Wherever Adam optimization algorithm~\citep{kingma2014adam} is used, its exponential decay rates for estimating the first and the second moments are set to 0.9 and 0.999 respectively (the default recommended by the original paper). The results are also verified for RAdam algorithm which is introduced more recently by~\citet{liu2019variance}. The definition of RAdam is directly borrowed from the implementation that the authors provided and the default hyper-parameters are preserved.

\subsubsection{Results and discussion}
Similar to Section~\ref{sec:exp_lr}, when the unified learning rate is scaled down or scaled up for the baselines, in this experiment, the ADP performance is either compromised at the beginning of fine-tuning or for a large number of steps afterward. In other words, optimizing the baselines with larger learning rates improves the performance at the beginning of fine-tuning but holds it back from reaching its potential performance thereafter. However, as it is shown in Figures~\ref{fig:optim_vgg_sgd} and~\ref{fig:optim_dense_sgd}, employing FAST provides with achieving a non-compromised ADP. Figures~\ref{fig:optim_vgg_adam} and~\ref{fig:optim_dense_adam}, have the same message as that of Figures~\ref{fig:optim_vgg_sgd} and~\ref{fig:optim_dense_sgd} except that all the models corresponding to the curves in these Figures use Adam optimizer. The baseline curves are marked with f/i which refers to the fan-in mode of Kaiming~\citep{He_2015_ICCV} method used for initializing the appended parameters.\footnote{In these cases, since $D > N$, choosing the fan-out mode of Kaiming method to initialize $\mW$ results in worse performance than that of the fan-in mode}

A rough comparison between Figures~\ref{fig:optim_vgg_sgd} and~\ref{fig:optim_vgg_adam}, clearly reveals that Adam substantially speeds-up the convergence of the baselines compared to SGD with momentum. However, applying Adam on top of FAST does not make a huge difference. To make it easier to compare the cases that use FAST in Figures~\ref{fig:optim_vgg_sgd} and~\ref{fig:optim_vgg_adam}, we plotted the same corresponding curves again in Figure~\ref{fig:optim_vgg_fast}. We also compared the aforementioned optimization algorithms that use FAST with a case in which Nesterov momentum~\citep{nesterov1983method} is used instead of the simple momentum in SGD. Also notice that as expected, Adam fails to get on par with SGD with momentum after a long period of training but generally the performance gap comparably looks minor during the whole training period.

Finally, Figure~\ref{fig:optim_dense_radam} compares FAST with the baselines when RAdam optimization algorithm~\citep{liu2019variance} is employed. Unlike the original setup used in ~\citet{liu2019variance}, in our fine-tuning setting the convergence rate of RAdam seems not to be less sensitive to the choice of the learning rate. Moreover, comparing Figure~\ref{fig:optim_adam} with Figure~\ref{fig:optim_radam} suggests that Adam converges faster than RAdam for the settings used in this experiment.

In summary, FAST could be used along with advance optimization algorithms to improve the ADP performance. However,~\textbf{when FAST is used for fine-tuning a model pretrained on a source task on a similar target task, the gap between performance obtained from different optimization algorithms is notably decreased}. This supports our hypothesis which suggests that FAST prevents the parameters to suddenly step out of the proximity of $\vtheta_{\gT_s}^*$ and $\vtheta_{\gT_t}^*$---Supposing $\vtheta_{\gT_s}^*$ and $\vtheta_{\gT_t}^*$ lay close to each other in a smooth region on $\gS(\vtheta, \ell)$ due to similarity of ${\gT_s}$ and ${\gT_t}$. 

\begin{figure}[p]
\centering       
    \begin{subfigure}{1.0\textwidth}
        \centering
        \def\svgwidth{14cm}
        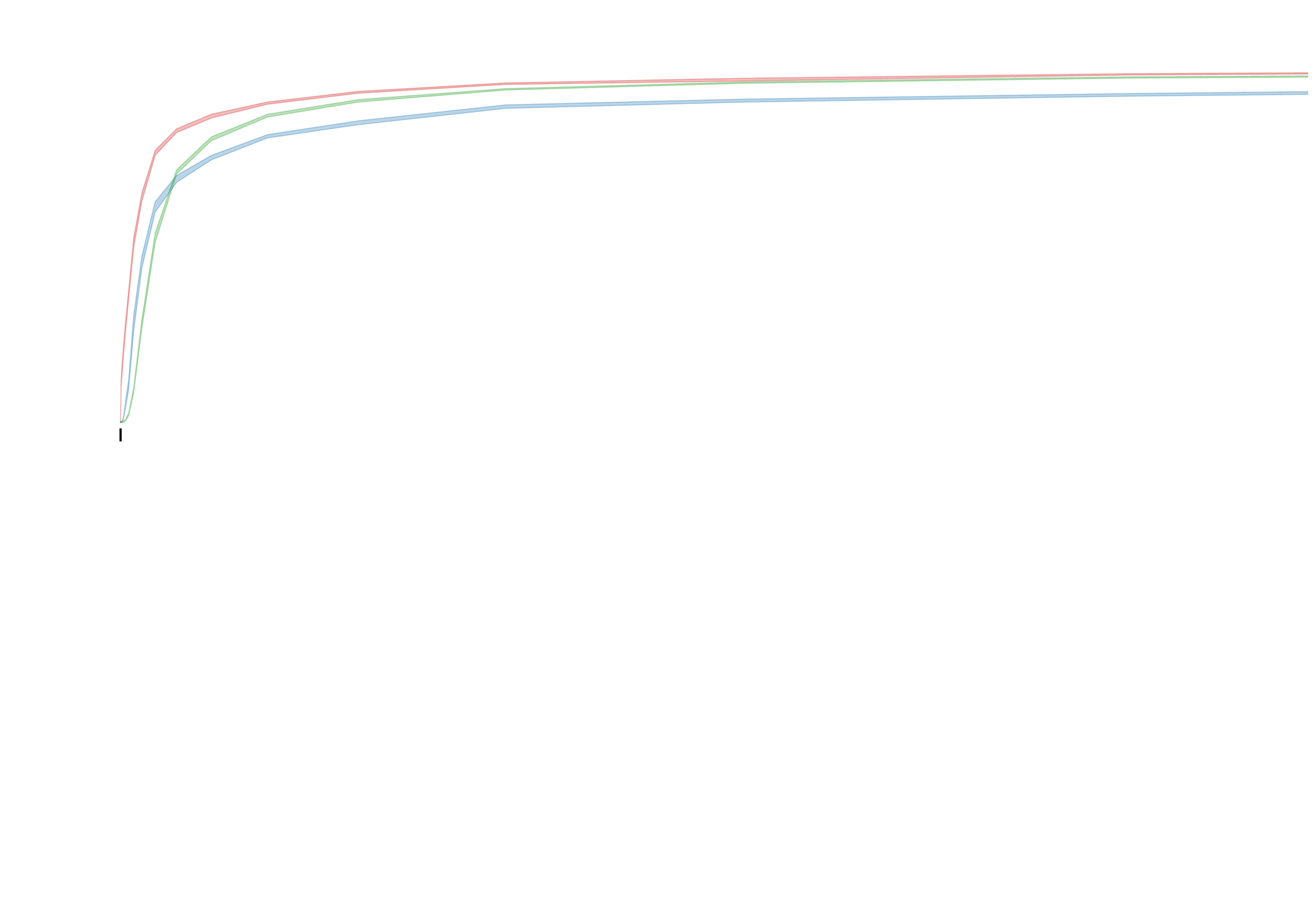
        \caption{\textsc{DenseNet-201}, $c=1\times10^{-2}$}
        \label{fig:optim_dense_sgd}
    \end{subfigure}
    \begin{subfigure}{1.0\textwidth}
        \centering
        \def\svgwidth{14cm}
        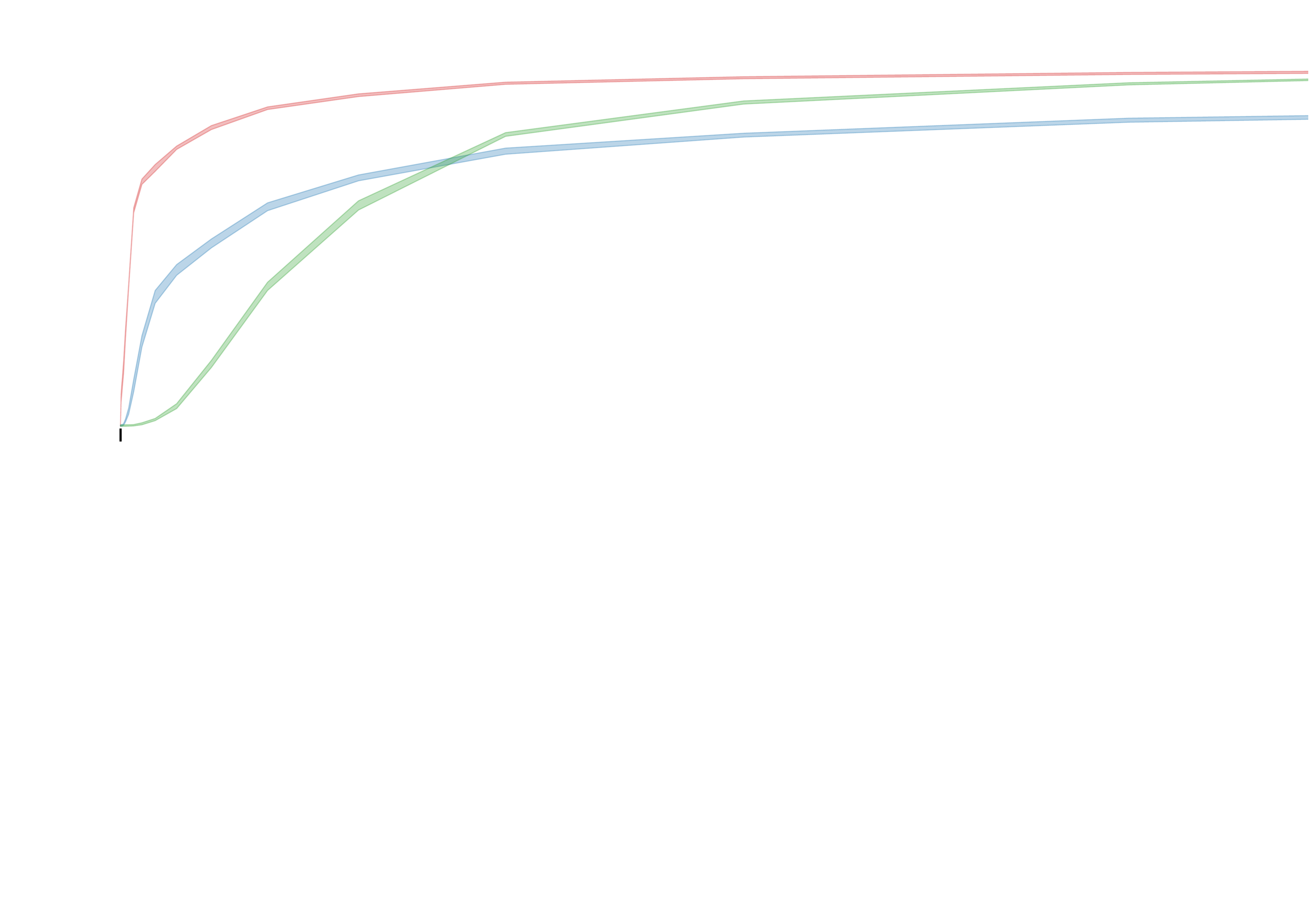
        \caption{\textsc{VGG-19}, $c=5\times10^{-2}$}
        \label{fig:optim_vgg_sgd}
    \end{subfigure}
    \caption{Comparison between FAST and the baselines in their learning progress when SGD with momentum is used as the optimization algorithm.}
    \label{fig:optim_sgd}
\end{figure}

\begin{figure}[p]
\centering       
    \begin{subfigure}{1.0\textwidth}
        \centering
        \def\svgwidth{14cm}
        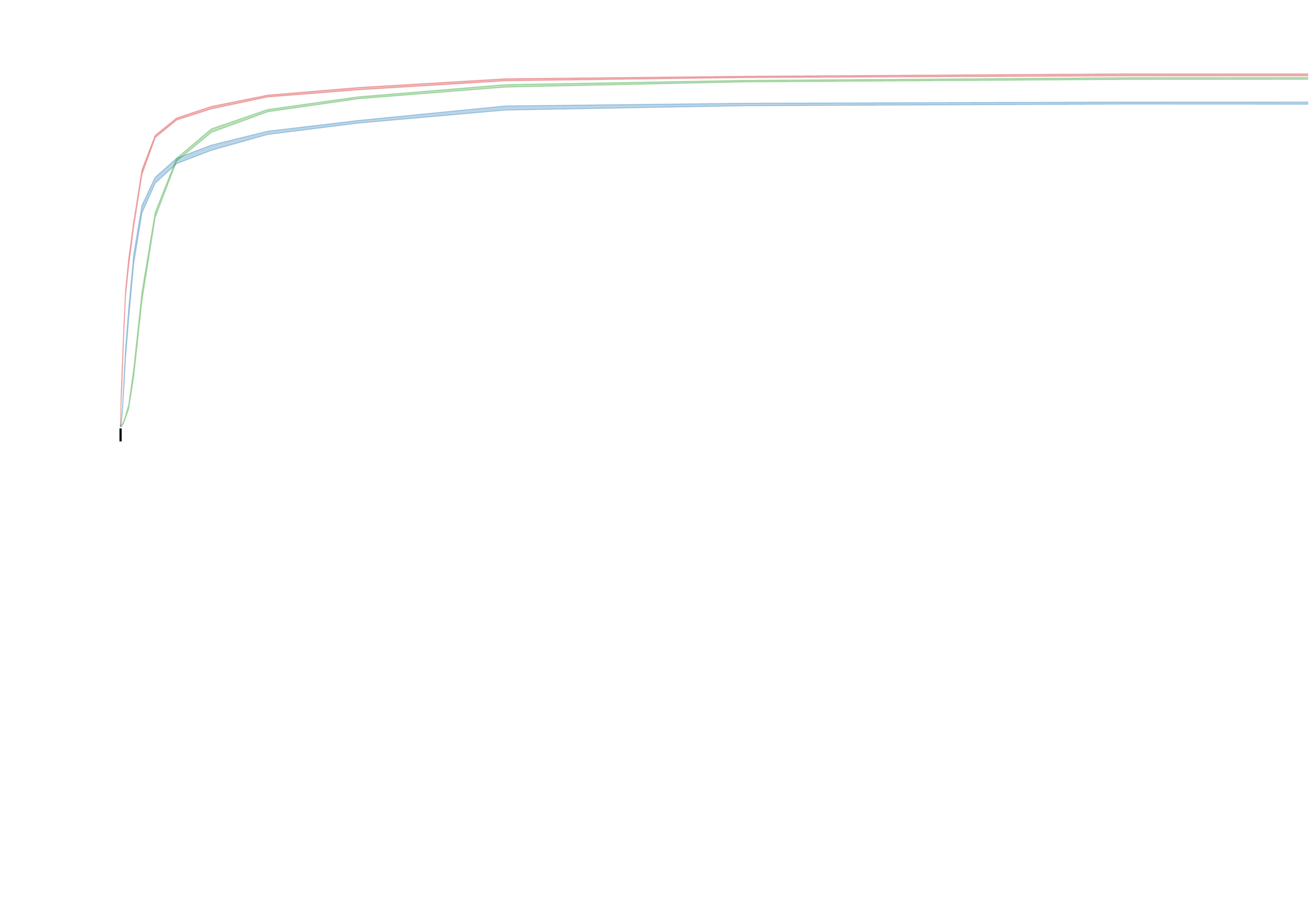
        \caption{\textsc{DenseNet-201}, $c=5\times10^{-4}$}
        \label{fig:optim_dense_adam}
    \end{subfigure}
    \begin{subfigure}{1.0\textwidth}
        \centering
        \def\svgwidth{14cm}
        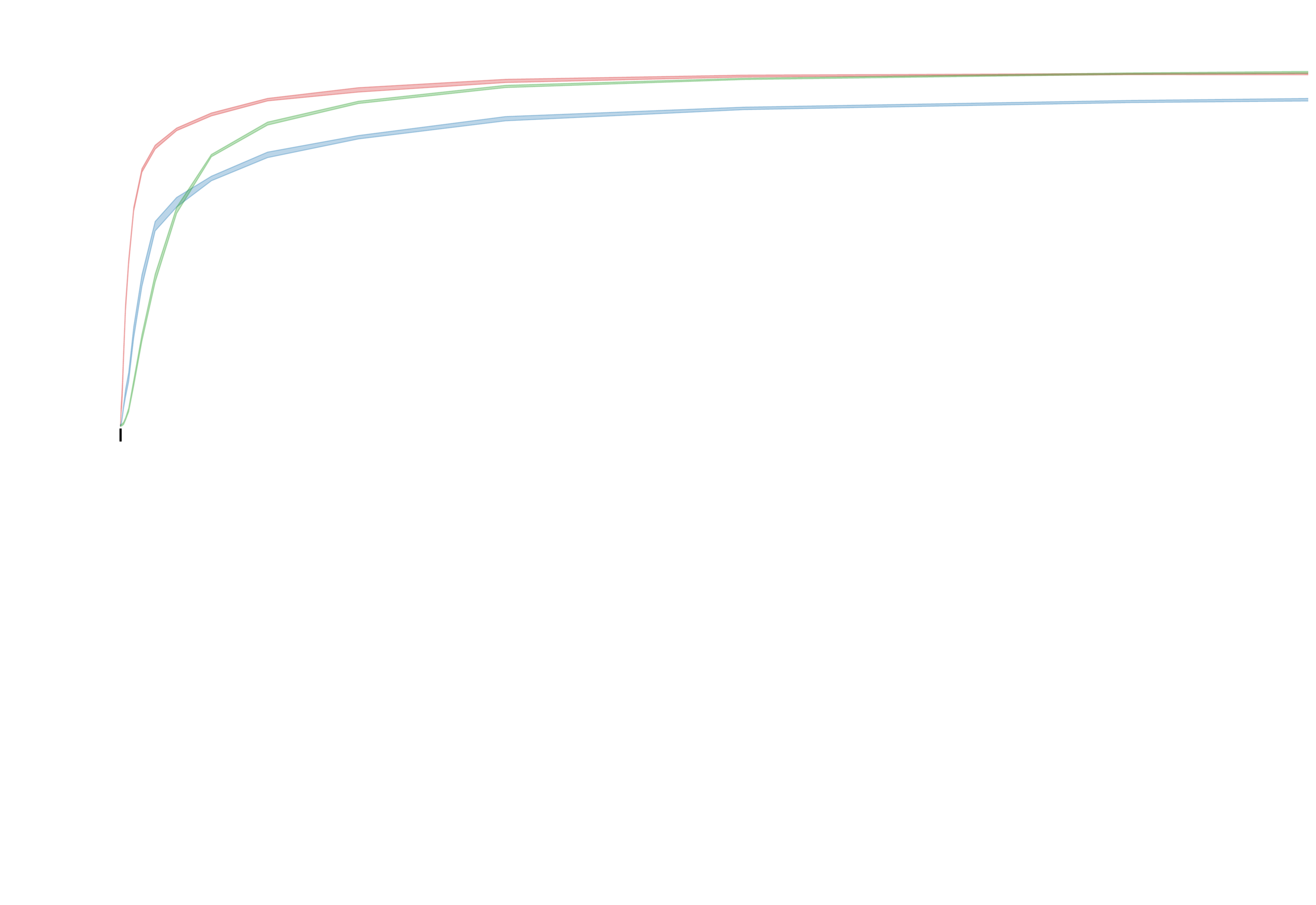
        \caption{\textsc{VGG-19}, $c=5\times10^{-4}$}
        \label{fig:optim_vgg_adam}
    \end{subfigure}
    \caption{Comparison between FAST and the baselines in their learning progress Adam is used as the optimization algorithm.}
    \label{fig:optim_adam}
\end{figure}

\begin{figure}[p]
\centering       
    \begin{subfigure}{1.0\textwidth}
        \centering
        \def\svgwidth{14cm}
        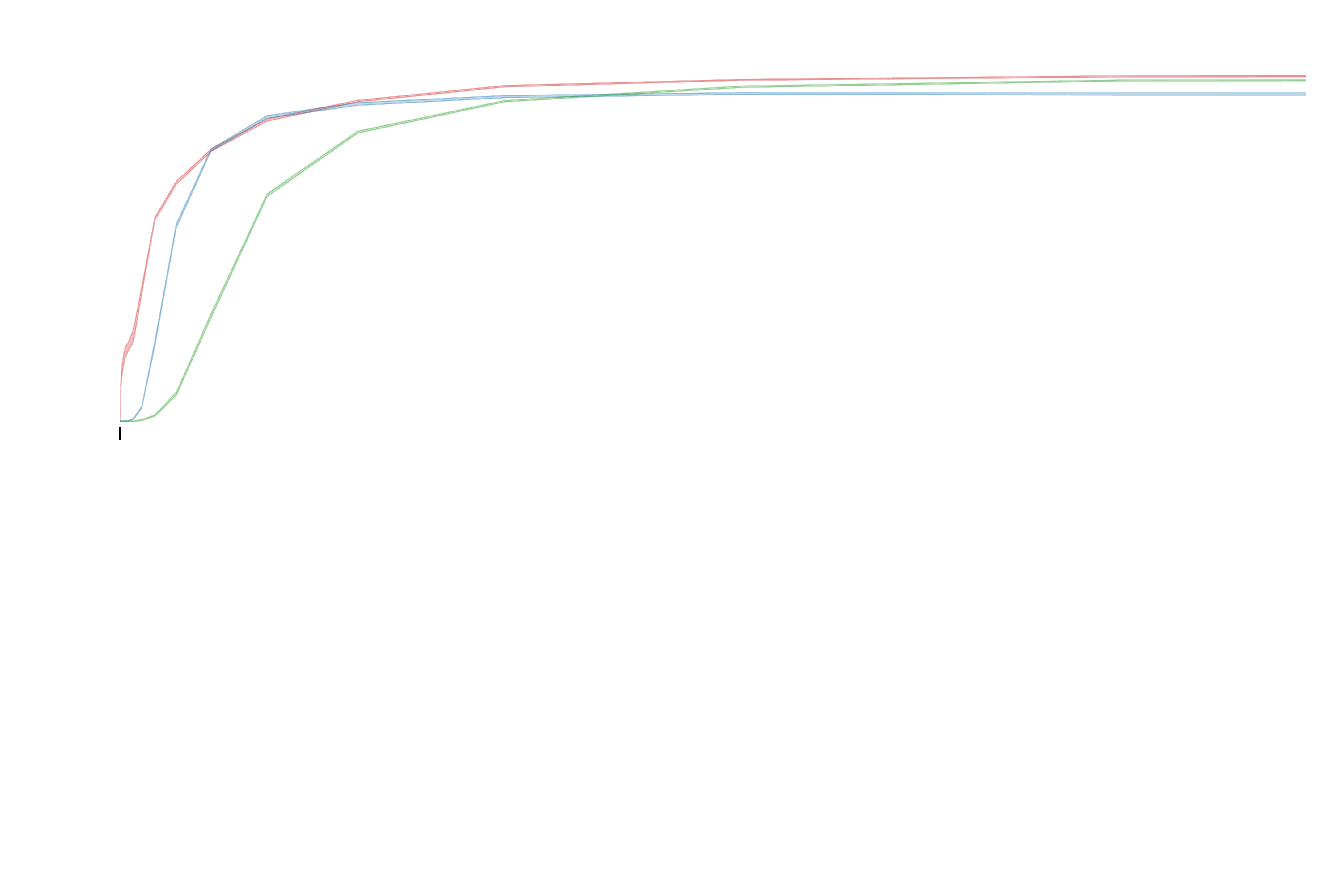
        \caption{\textsc{DenseNet-201}, $c=5\times10^{-4}$}
        \label{fig:optim_dense_radam}
    \end{subfigure}
    \begin{subfigure}{1.0\textwidth}
        \centering
        \def\svgwidth{14cm}
        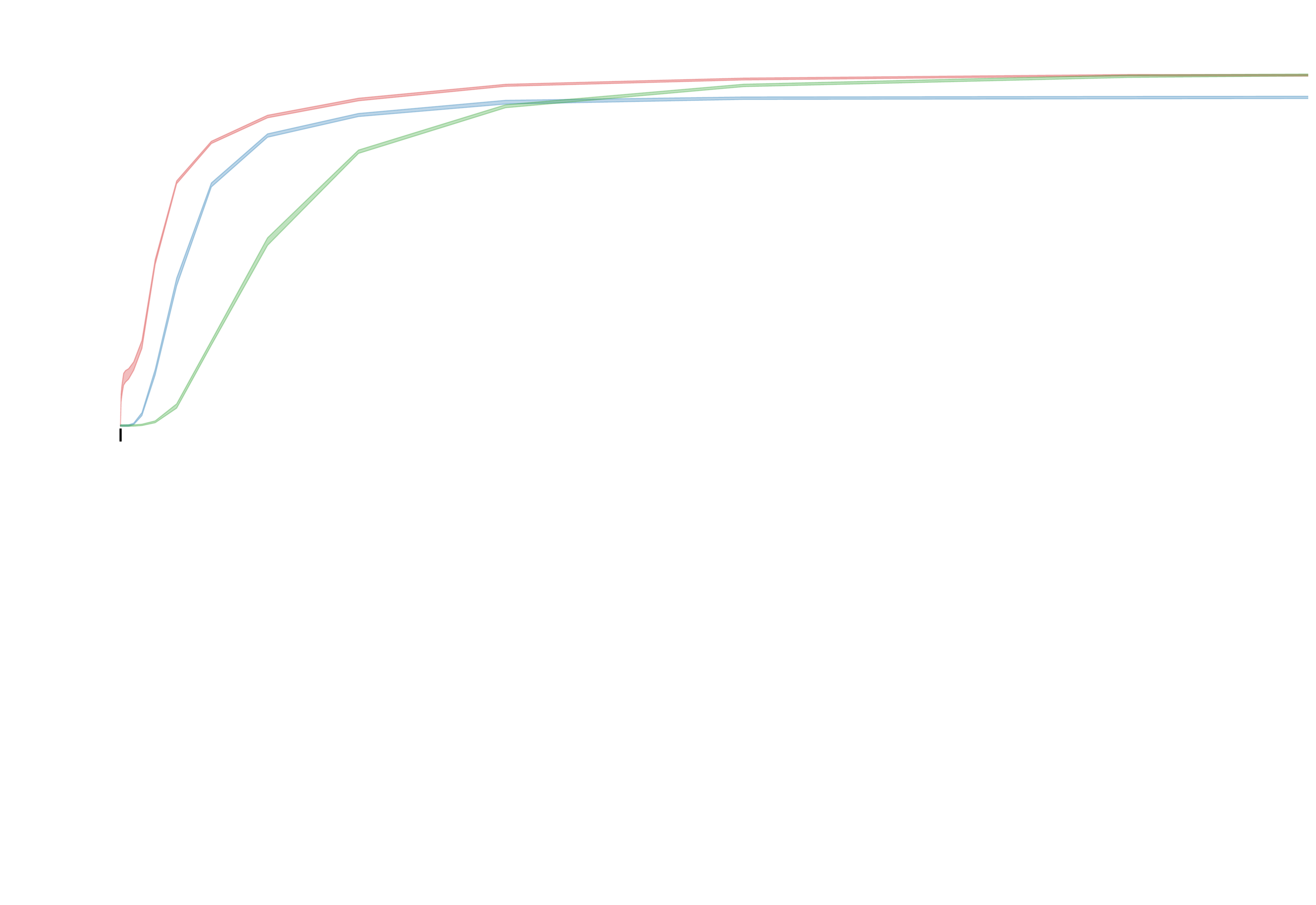
        \caption{\textsc{VGG-19}, $c=5\times10^{-4}$}
        \label{fig:optim_vgg_radam}
    \end{subfigure}
    \caption{Comparison between FAST and the baselines in their learning progress when RAdam is used as the optimization algorithm.}
    \label{fig:optim_radam}
\end{figure}

\begin{figure}[p]
\centering       
    \begin{subfigure}{1.0\textwidth}
        \centering
        \def\svgwidth{14cm}
        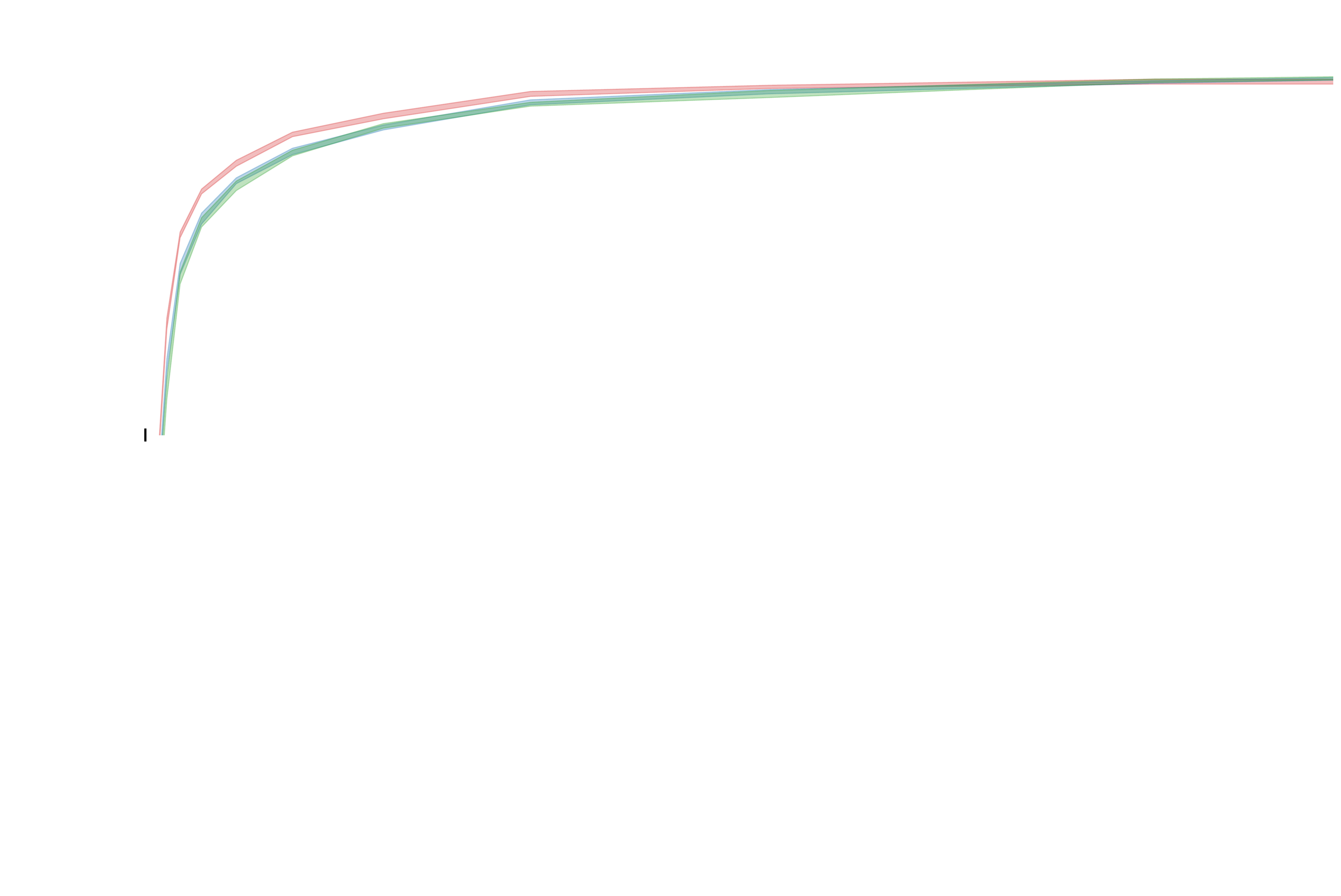
        \caption{\textsc{Net-201}}
        \label{fig:optim_dense_fast}
    \end{subfigure}
    \begin{subfigure}{1.0\textwidth}
        \centering
        \def\svgwidth{14cm}
        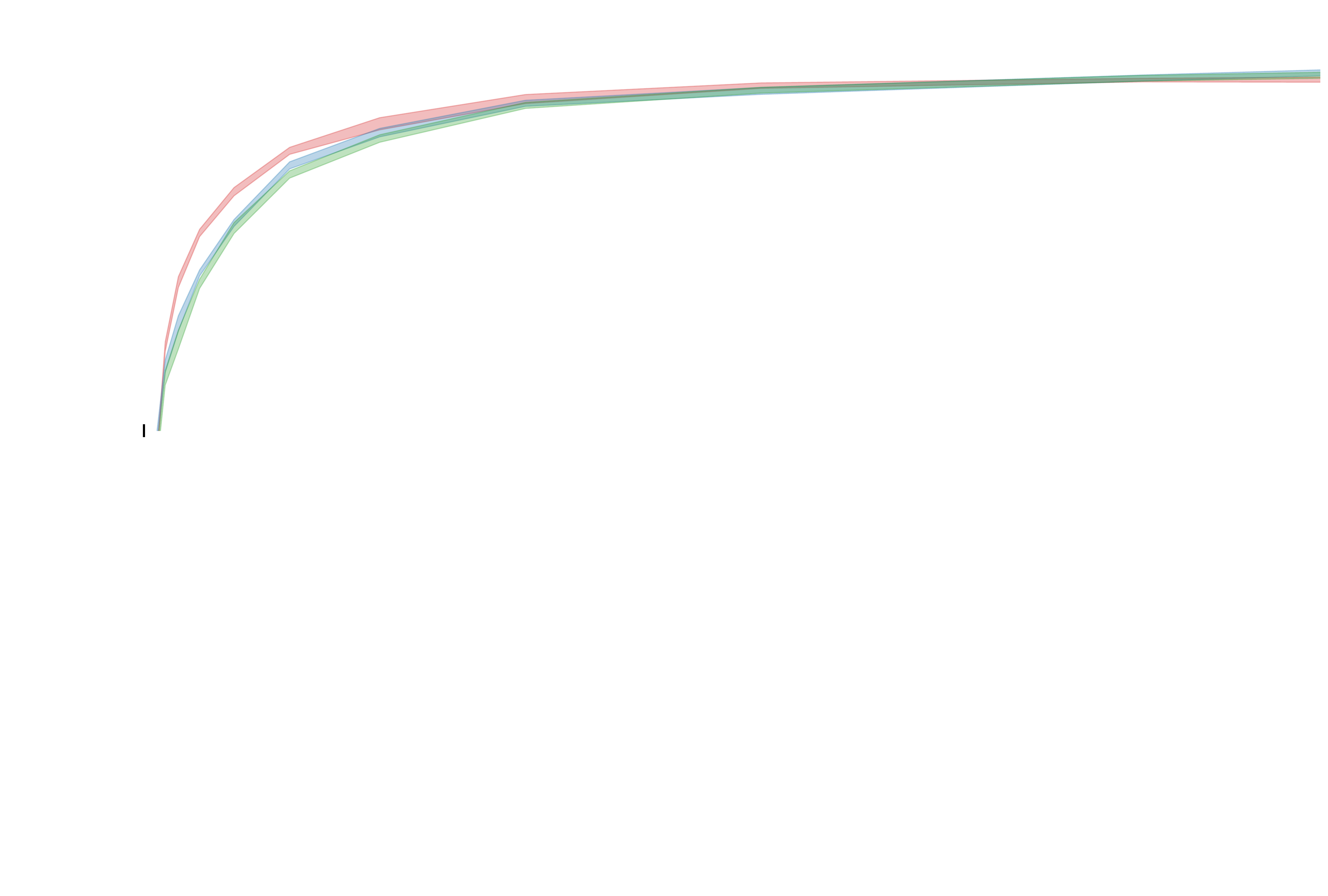
        \caption{\textsc{VGG-19}}
        \label{fig:optim_vgg_fast}
    \end{subfigure}
    \caption{Comparison between the learning progresses corresponding to different optimization algorithms used with FAST.}
    \label{fig:optim_fast}
\end{figure}

\subsection{Convergence Accuracy}
In this experiment we compare FAST and the baselines with a focus on the convergence accuracy. The main goal is to confirm that using FAST would not have an adverse effect on the convergence performance; therefore, the models in this experiment are fine-tuned for a longer period than those in the previous experiments. Another goal of this experiment is to quantitatively confirm that fine-tuning with FAST keeps $\vtheta$ closer to $\vtheta_{\gT_s}^*$ compared to when the baselines are employed. This is similar to the experiment conducted in Section~\ref{sec:exp_init} in that they both indicate how much our method and the baselines are resistant to forgetting. However, the one in Section~\ref{sec:exp_init} focuses on the generalizable accuracy while in this experiment the geometric perspective is highlighted. 

\subsubsection{Setup}
For this experiment we employ~\textsc{ResNeXt-101-32x8d} introduced by~\citet{xie2017aggregated}, which is already pretrained on ImageNet's~\textsc{ILSVRC-12}~\citep{russakovsky2015imagenet}. The fine-tuning is done on~\textsc{Cifar-100}~\citep{krizhevsky2009learning}, 
and~\textsc{Caltech-256}~\citep{griffin2007caltech} classification tasks.~\textsc{Caltech-256} contains 74$\times$74 labeled natural  RGB images from 256 classes. It is originally published with no train-test split; therefore, we made a random but stratified split with 20\% of the images as test (the split could be found in our code). In the training split, 5 images are randomly selected from each class to form the validation set. This quantity is aligned with the size of validation set we chose for the other data sets (5\% of the whole training) employed in this paper. 
The images in~\textsc{Caltech-256} data set are resized to twice of their original size in both their width and their height. 
Fine-tuning on~\textsc{CIFAR-100} and~\textsc{Caltech-256} continued for 250 and 330 epochs respectively. In all the cases, the batch size is set equal to the number of classes (that is $M=N$), and the training batches are sampled in an stratified order.

\subsubsection{Results and discussion}
Figures~\ref{fig:c100} and~\ref{fig:cal256} compare the top-1 ADP obtained from FAST and the baselines. The wide plots on the left of each Figure show the ADP of the first SGD steps while the narrow one on the right side magnify the same quantity but for a larger number of steps taken afterward. Neither the magnitudes nor the ratios of the axes are preserved among the plots in each Figure. The confidence interval is shown with transparent shadows but are barely visible since there is a high consistency among the results from different random seeds. The convergence results are also restated in numbers in Table~\ref{tab:converged}.~\textbf{The ADP results suggest that not only FAST learns faster than the baselines but it also helps reaching higher convergence performance.} The latter is not a very strong argument in the sense that the training time is subjective and the point of convergence has no formal meaning in this context---i.e., one may claim that after an infinite number of training steps the difference between ADP results may become insignificant~\citep{He_2019_ICCV}.

Assume that after the chosen number of training epochs, the convergence is achieved (in its liberal definition: no improvement for a fairly large number of epochs).
In the last column of Table \ref{tab:converged} we present the euclidean distance travelled by the feature extractor's parameters from aiming minimum of the source task's loss to that of the target task's loss. The results confirm that compared to the baselines, FAST preserves $\vtheta_{\gT_t}^*$ in a closer proximity to $\vtheta_{\gT_s}^*$. This implies that~\textbf{FAST reaches the objective of the target task from a direction that is more aligned with the objective of the source task compared to the baselines.} This characteristic makes FAST more resistant to catastrophic forgetting.

\begin{figure}[t]
    \centering
    \def\svgwidth{15cm}
    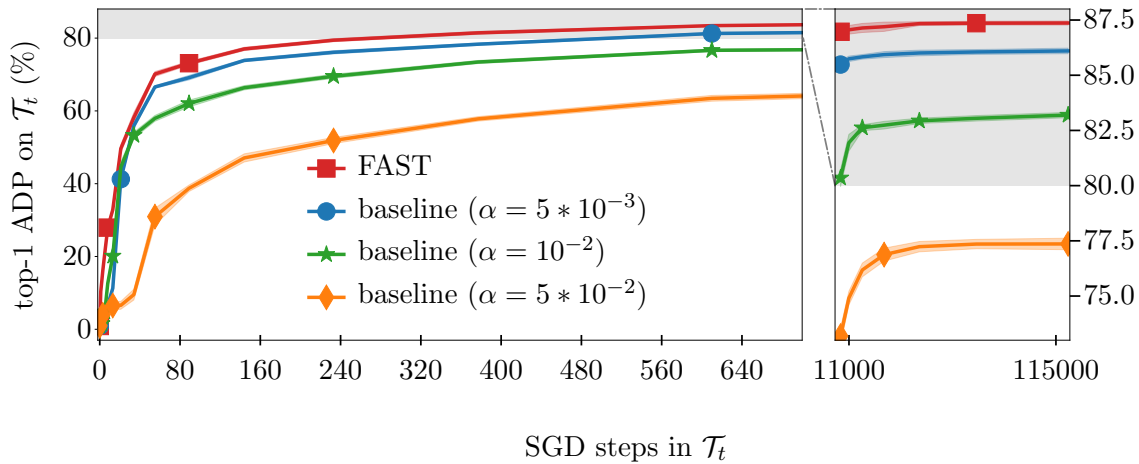
    \caption{Comparing the convergence ADP of FAST and the baselines on \textsc{CIFAR-100} data set.}
    \label{fig:c100}
\end{figure}

\begin{figure}[t]
    \centering
    \def\svgwidth{15cm}
    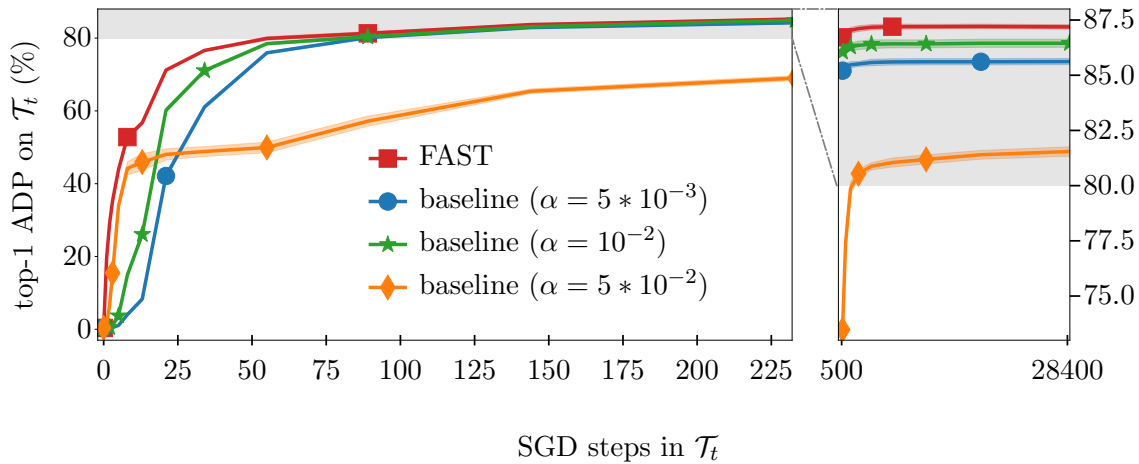
    \caption{Comparing the convergence ADP of FAST and the baselines on \textsc{Caltech-256} data set.}
    \label{fig:cal256}
\end{figure}

\begin{table}[t]
    \centering
    \begin{tabular}{l l c  c  c}
    \toprule
        data set & method & top-1 error & top-5 error & ${||\vtheta_{\gT_t}^* - \vtheta_{\gT_s}^*||}_2$\\ \midrule
        \multirow{3}{7em}{CIFAR-100}    
         & baseline ($\alpha=5*10^{-2}$) & 22.64 $\pm$ 0.26 & 5.91  $\pm$ 0.14 & 119.51 $\pm$ 1.38 \\
         & baseline ($\alpha=10^{-2}$) & 16.80 $\pm$ 0.12 & 3.46 $\pm$ 0.06 & 34.58 $\pm$ 0.43 \\ 
         & baseline ($\alpha=5*10^{-3}$) & 13.90  $\pm$ 0.13 & 2.22 $\pm$ 0.08 & 14.32 $\pm$ 0.12 \\ 
         & FAST & \textbf{12.63} $\pm$ \textbf{0.09} & \textbf{1.62} $\pm$ \textbf{0.04} & \textbf{7.19} $\pm$ \textbf{0.04} \\
        \midrule[\heavyrulewidth]
        \multirow{3}{7em}{\textsc{Caltech-256}}
         & baseline ($\alpha=5*10^{-2}$) & 18.45 $\pm$ 0.14 & 7.14 $\pm$ 0.01 & 37.2 $\pm$ 0.23 
         \\
         & baseline ($\alpha=10^{-2}$) & 13.55 $\pm$ 0.19 & 4.38 $\pm$ 0.09 & 7.70 $\pm$ 0.07 
         \\
         & baseline ($\alpha=5*10^{-3}$) & 14.39 $\pm$ 0.14 & 4.73 $\pm$ 0.10 & 5.93 $\pm$ 0.17
         \\
         & FAST & \textbf{12.81} $\pm$ \textbf{0.10} & \textbf{3.96} $\pm$ \textbf{0.08} & \textbf{5.28} $\pm$ \textbf{0.06} \\
        \bottomrule
        
    \end{tabular}
    \caption{Comparing FAST with baselines in their convergence performance and traveled distance in $\gS(\vtheta)$ from $\vtheta_{\gT_s}^*$. In all baselines, the classifier is initialized using the fan-in mode of Kaiming method~\citep{He_2015_ICCV} and $\alpha=\beta$.}
    \label{tab:converged}
\end{table}

\section{Conclusion and Future Works}
\label{sec:future}
In this paper we pin-pointed a major problem of traditional fine-tuning that not only leads to catastrophic forgetting, but also slows down the convergence to the target task. To clarify this problem, we provided a novel perspective of the loss landscape which makes it possible to connect the loss landscapes of the source and target tasks. Using this perspective, we compared different transfer-learning techniques and proposed FAST for fine-tuning pretrained neural networks on classification tasks. With a focus on machine vision and deep CNN models, we provided empirical analysis with the following outcomes
\begin{itemize}
    \item Initializing the classifier appended through model adoption with close-to-zero values, not only better preserves the transferred knowledge but also accelerates the training procedure on the target task (Section \ref{sec:exp_init}).
    \item The convergence speed is largely influenced by the randomness of the parameters of the classifier and how quickly they are adapted (Section \ref{sec:exp_fix_space}).
    \item Although including a classifier warmup phase (as suggested by~\citet{li2018learning}) can help the optimization algorithm to take its initial steps in a correct direction, it still is prone to minimum overshooting (Section \ref{sec:exp_wup}).
    \item For an efficient convergence, the learning rates chosen for the pretrained feature-extractor and the appended classifier are better to be tuned separately. Otherwise, if they are chosen equally as in the traditional fine-tuning the anytime performance could be compromised either in the beginning of the fine-tuning or during a large number of optimization steps afterward (Section \ref{sec:exp_lr}).
    \item Compared to the traditional fine-tuning, using FAST 
    remarkably decreases the gap between the performance obtained from SGD with momentum and Adam optimization algorithms.
\end{itemize}
Generally, our proposed method learns the target task faster than the traditional fine-tuning besides preserving more knowledge about the source task. We showed superiority of our method from different analytical and empirical viewpoints, though yet more aspects have remained to explore. For instance, one can investigate if compared to the traditional fine-tuning, FAST converges with a different level of \textit{inductive bias} as contextualized by~\citep{abnar2020transferring}. 

A similar strategy taken in this paper can be applied to improve the ADP of pretrained DNNs employed in tasks such as regression, detection, segmentation, etc. To do so, one needs to first generalize the idea so it can work with other loss functions. This sounds challenging since a key component in our analysis to design FAST is the particular behaviour of the softmax function. Therefore, some modifications may be necessary before taking advantage of our strategy for each particular loss function.
Moreover, although FAST is only applied on image classification tasks in this paper, it has no ties to the type of input data. Therefore, we hypothesize that FAST would show similar level of superiority over the baselines when knowledge is transferred between models that operate on other types of data such as text.
Furthermore, 
due to its level of resistance to forgetting, FAST can have a positive contribution to the performance and efficiency of the state-of-the-art methods for continual learning.
Finally, we consider FAST an appealing choice for improving the efficiency of the compute-intensive machine learning processes such as Automatic Machine Learning (\textit{Auto ML}) and, in particular, neural architecture and hyper-parameter search~\citep{zela2018towards}.

\newpage
\appendix
\section{}
    

\label{appendix:append}

\begin{lemma}
\label{th:yhat_norm_sup}
If vector $\vv$ is the result of feeding vector $\vu \in \R^N$ to a logistic softmax, then $\frac{1}{N}$ is the supremum of ${||\vv||}_2^2$.
\end{lemma}

\begin{proof}
Using the definition of logistic softmax (\EqrefNP{eq:sm}), we have
\begin{equation*}
{||\vv||}_2^2 = \sum_{n=1}^{N}  \frac{ e^{2\vu_n}}{(\sum_{n'=1}^{N} e^{\vu_{n'}})^2} = \frac{\sum_{n=1}^{N} e^{2\vu_n}}{(\sum_{n'=1}^{N} e^{\vu_{n'}})^2}.
\end{equation*}
Taking the partial derivative of ${||\vv||}_2^2$ with respect to the $i$-th element of the $\vu$ gives
\begin{equation*}
\begin{split}
    &\frac{\partial {||\vv||}_2^2}{\partial \vu_i} =\\ &\frac{2e^{2\vu_i} (\sum_{n'=1}^{N} e^{\vu_{n'}})^2 - 2e^{\vu_{i}} (\sum_{n'=1}^{N} e^{\vu_{n'}}) (\sum_{n=1}^{N} e^{2\vu_n})}{(\sum_{n'=1}^{N} e^{\vu_{n'}})^4}
\end{split}
\end{equation*}
which could be simplified to
\begin{equation*}
\begin{split}
    \frac{\partial {||\vv||}_2^2}{\partial \vu_i} =\frac{2e^{2\vu_i} \sum_{n'=1}^{N} e^{\vu_{n'}} - 2e^{\vu_{i}} (\sum_{n=1}^{N} e^{2\vu_n})}{(\sum_{n'=1}^{N} e^{\vu_{n'}})^3}.
\end{split}
\end{equation*}
Setting the calculated partial derivative to zero yields the value of $\vu_i$ for which ${||\vv||}_2^2$ is maximized in that direction. That is if
\begin{equation*}
    \frac{\partial {||\vv||}_2^2}{\partial \vu_i} = 0
\end{equation*}
then,
\begin{equation}
\label{eq:equal_sums}
    \sum_{n'=1}^{N} e^{\vu_{i}} e^{\vu_{n'}} =\sum_{n'=1}^{N} e^{2\vu_{i}}.
\end{equation}
re-indexing the sum on the left side of \Eqref{eq:equal_sums} by $n$, results in
\begin{equation*}
    \forall \ n \in \R^N: \vu_n = \frac{1}{N}
\end{equation*}

\end{proof}

\begin{lemma}
\label{th:yhat_norm_sm}
If elements of a vector $\vv \in \R^N$ sum to 1, then $\frac{1}{N}$ is the supremum of ${||\vv||}_2^2$.
\end{lemma}

\begin{proof}
Using Cauchy-Schwarz inequality,
\begin{equation*}
\left(\vv\  \vone_N    \right)^2 \leq \left(\vone_N^{T}\  \vone_N \right)
\left(\vv\ \vv^{T}\right)\ ,
\end{equation*}
and since $\vv\  \vone_N=1$,
\begin{equation*}
    \vv\ \vv^{T} = {||\vv||}_2^2 \geq \frac{1}{N}.
\end{equation*}
\end{proof}

\begin{axiom}
\label{th:sum_max}
Sum of a set of all non-negative random variables is maximized when each of the entries meet its maximum.
\end{axiom}

\bibliography{egbib}

\end{document}